%% file: ms.tex
\documentclass[a4paper,runningheads,envcountsame]{llncs}
\usepackage[utf8]{inputenc}
\usepackage[T1]{fontenc}
\usepackage[english]{babel}

\usepackage{framed}
\usepackage{lmodern}
\usepackage{amssymb}
\usepackage{amsmath}
\usepackage{todonotes}
\usepackage{hyperref}
\usepackage{doi}
\usepackage[vlined,ruled,noend,linesnumbered]{algorithm2e}
\SetArgSty{text}
\DontPrintSemicolon
\usepackage{bussproofs}
\EnableBpAbbreviations
\usepackage{tikz}
\usetikzlibrary{shapes.geometric, patterns}
\usetikzlibrary{arrows.meta}
\usetikzlibrary{calc}
\usetikzlibrary{shapes.misc, positioning}
\usepackage{csquotes}
\usepackage{cite,doi}
\usepackage{booktabs}
\usepackage{thmtools}
\usepackage{thm-restate}
\usepackage{enumitem}
\usepackage{microtype}
\usepackage{forest}
\usepackage{colortbl}
\usepackage{multirow}

\usepackage[hyphenbreaks]{breakurl} %

\usepackage{xcolor}
\usepackage{pgfplots}
\pgfplotsset{compat=1.16}%

\usepackage{xparse}

\usepackage{stmaryrd}

\usepackage{slashbox}

\usepackage{cancel}

\input{macros}

\newif\ifhideproofs

\newif\iftechnicalReport

\newif\ifplotDataInTikz

\newif\ifappendix

\plotDataInTikztrue %
\hideproofsfalse
\technicalReporttrue %

\ifhideproofs
  \usepackage{environ}
  \NewEnviron{hide}{}

\fi

\iftechnicalReport
\title{Finding Good Proofs for Description Logic Entailments Using Recursive Quality Measures (Extended Technical Report)}
\else
\title{Finding Good Proofs for Description Logic Entailments Using Recursive Quality Measures}
\fi
\titlerunning{Good Proofs for DL Entailments}
\author{
  Christian~Alrabbaa\orcidID{0000-0002-2925-1765} \and 
  Franz~Baader\orcidID{0000-0002-4049-221X} \and
  Stefan~Borgwardt\orcidID{0000-0003-0924-8478} \and  
  Patrick~Koopmann\orcidID{0000-0001-5999-2583} \and 
  Alisa~Kovtunova\orcidID{0000-0001-9936-0943}}
\authorrunning{Alrabbaa, Baader, Borgwardt, Koopmann, Kovtunova}
\institute{Theoretical Computer Science, TU Dresden, Dresden, Germany}

\newcommand{\added}[1]{#1}%

\begin{document}

\maketitle

\begin{abstract}
Logic-based approaches to AI have the advantage that their behavior can in principle be explained to 
a user.
If, for instance, a Description Logic  reasoner derives a consequence that triggers some action of
the overall system, then one can explain such an entailment by presenting a proof of the consequence
in an appropriate calculus. How comprehensible such a proof is depends not only on the employed calculus,
but also on the properties of the particular proof, such as its overall size, its depth, the complexity
of the employed sentences and proof steps, etc. For this reason, we want to determine
the complexity of generating proofs that are below a certain threshold w.r.t.\ a given measure
of proof quality. Rather than investigating this problem for a fixed proof calculus and a fixed measure,
we aim for general results that hold for wide classes of calculi and measures. In previous work,
we first restricted the attention to a setting where proof size is used to measure the quality
of a proof. We then extended the approach to a more general setting, but important measures such as
proof depth were not covered. In the present paper, we provide results for a class of measures called
recursive, which yields lower complexities and also encompasses proof depth. In addition, we close some gaps  left open
in our previous work, thus providing a comprehensive picture of the complexity landscape.
\end{abstract}

\input{introductionCADE}

\input{complexity-table}

\input{preliminaries}

\input{proofs-and-derivation-structures}

\input{measuring-proofs}
\input{monotone-recursive-measures}

\input{complexity-results}

\input{conclusion}

\paragraph{Acknowledgements} This work was supported by the DFG in grant 389792660
as part of TRR~248 (\url{https://perspicuous-computing.science}), and QuantLA, GRK 1763 (\url{https://lat.inf.tu-dresden.de/quantla}).

\def\doi#1{\url{https://doi.org/#1}}
\bibliographystyle{splncs04}%

\iftechnicalReport

\input{appendix.tex}

\fi

\end{document}

%% file: macros.tex
\newcommand{\ttodo}[4]{\ifthenelse{\equal{#1}{inline}}{\todo[inline, author=#2, color =
#3]{#4}}{\todo[color=#3]{#2: #4}}}

\newcommand{\wrt}{w.r.t.\ }
\newcommand{\st}{s.t.\ }
\newcommand{\ie}{i.e.\ }
\newcommand{\eg}{e.g.\ }
\newcommand{\cf}{cf.\ }

\newlength{\myl}
\newcommand{\longsquigarrow}[1]{
    \settowidth{\myl}{$~_{#1}$}
    \raisebox{-0.01cm}{\xymatrix@C=\myl{
            {}\ar@{~>}[r]^{~_{#1}}&{}
        }
    }
}

\newcommand{\OP}{\ensuremath{\mathsf{OP}}\xspace}

\newcommand{\un}{\ensuremath{\mathsf{unary}}\xspace}
\newcommand{\bin}{\ensuremath{\mathsf{binary}}\xspace}

\def\define#1#2#3%
{%
\renewcommand*{\do}[1]{%
\expandafter\newcommand\csname
#1\endcsname{#2}
}
\docsvlist{#3}
}

\define{#1}
{{\text{\upshape{\textsc{#1}}}}\xspace}
{NP,coNP,PSpace,ExpTime,ExpSpace,NLogTime,NExpTime,LogSpace}

\newcommand{\PTime}{\text{\upshape{\textsc{P}}}\xspace}

\newcommand{\card}[1]{\lvert #1\rvert}
\newcommand{\poly}{\ensuremath{\mathsf{polynomial}}\xspace}
\renewcommand{\exp}{\ensuremath{\mathsf{exponential}}\xspace}
\newcommand{\reasoner}{deriver\xspace}
\newcommand{\Reasoner}{Deriver}
\newcommand{\reasoners}{derivers\xspace}
\newcommand{\Reasoners}{Derivers}
\newcommand{\pis}{\ensuremath{\mathrm{\Phi}}\xspace}
\newcommand{\si}{\textbf{\upshape{[HI]}}\xspace}

\newcommand{\pol}{\textbf{\upshape{[P]}}\xspace}

\newcommand{\Lmc}{\ensuremath{\mathcal{L}}\xspace}

\newcommand{\Smc}{\ensuremath{\mathcal{S}}\xspace}
\newcommand{\Tmc}{\ensuremath{\mathcal{T}}\xspace}

\newcommand{\EL}{\ensuremath{\mathcal{E}\hspace{-0.1em}\mathcal{L}}\xspace}

\newcommand{\ELI}{\ensuremath{\mathcal{ELI}}\xspace}
\renewcommand{\L}{\ensuremath{\mathcal{L}}\xspace}

\newcommand{\Rationals}{\ensuremath{\mathbb{Q}}\xspace}

\newcommand{\tup}[1]{(#1)}

\newcommand{\NC}{\ensuremath{\textsf{N}_\textsf{C}}\xspace}
\newcommand{\NR}{\ensuremath{\textsf{N}_\textsf{R}}\xspace}

\newcommand{\exFont}[1]{\ensuremath{\mathsf{#1}}\xspace}

\newcommand{\p}{\ensuremath{\mathcal{P}}\xspace}
\newcommand{\q}{\ensuremath{\mathcal{Q}}\xspace}

\newcommand{\R}{\ensuremath{\mathfrak{D}}\xspace}

\newcommand{\el}{\ensuremath{\ell}\xspace}
\newcommand{\m}{\ensuremath{\mathfrak{m}}\xspace}

\newcommand{\ds}{\ensuremath{\mathcal{D}}\xspace}

\newcommand{\true}{\ensuremath{\mathsf{true}}\xspace}
\newcommand{\false}{\ensuremath{\mathsf{false}}\xspace}

\DeclareDocumentEnvironment{block}{O {} }{\noindent \textbf{#1} \quad }{~\\}

\newcommand{\leaf}{\ensuremath{\mathsf{leaf}}\xspace}
\newcommand{\edge}{\ensuremath{\mathsf{edge}}\xspace}

\newcommand{\blank}{\cancel{b}}

\newcommand{\mtree}{\ensuremath{\m_{\mathsf{tree}}}\xspace}

\newcommand{\mdepth}{\ensuremath{\m_{\mathsf{depth}}}\xspace}
\newcommand{\msize}{\ensuremath{\m_{\mathsf{size}}}\xspace}

\newcommand{\mldepth}{\ensuremath{\m_{\mathsf{log(depth)}}}\xspace}

\newcommand{\quant}{{\mathsf{Q}}}

\newcommand{\Elk}{\text{\upshape{\textsc{Elk}}}\xspace}
\newcommand{\Eli}{\text{\upshape{\textsc{Eli}}}\xspace}

\newcommand{\Protege}{\text{Prot\'eg\'e}\xspace}

%% file: introductionCADE.tex
\section{Introduction}

Explainability has developed into a major issue in Artificial Intelligence,
particularly in the context
of sub-symbolic approaches based on Machine Learning \cite{XAIpaper}. In contrast, results produced by
symbolic approaches based on logical reasoning are ``explainable by design'' since a derived consequence can
be formally justified by showing a proof for it. In practice, things are not that easy since proofs may be
very long, and even single proof steps or stated sentences may be hard to comprehend 
for a user that is not an expert
in logic. For this reason, there has been considerable work in the Automated Deduction and Logic in AI communities
on how to produce ``good'' proofs for certain purposes, both for full first-order logic, but also for
decidable logics such a Description Logics (DLs) \cite{baader_horrocks_lutz_sattler_2017}. We mention here only a few approaches, and refer the reader
to the introduction of our previous work \cite{LPAR23:Finding_Small_Proofs_for} for a more detailed review.

First, there is work that transforms proofs that are produced by an automated reasoning system into ones
in a calculus that is deemed to be more appropriate for human
consumption~\cite{DBLP:conf/ijcai/Lingenfelder89,DeMc-96,DBLP:conf/ecai/BorgidaFH00}.
Second, abstraction techniques are used to reduce the size of proofs
by introducing definitions, lemmas, and more abstract deduction rules~\cite{10.5555/648231.752805,DBLP:conf/semweb/HorridgePS10}.
Justification-based explanations for DLs \cite{ScCo03,BaaSun-KRMED-08,Horr-11} can be seen as a radical abstraction technique where the
abstracted proof consists of a single proof step, from a minimal set of stated sentences 
that implies a certain consequence directly to this
consequence.
Finally, instead of presenting proofs in a formal, logical syntax, one can also try to increase readability by
translating them into natural language
text~\cite{DBLP:conf/birthday/Fiedler05,DBLP:conf/dlog/SchillerG13,DBLP:conf/esws/NguyenPPW13,DBLP:conf/dlog/SchillerSG17}
or visualizing them \cite{AlBaDaFlKo-DL-20}.

The purpose of this work is of a more (complexity) theoretic nature. We want to investigate how hard it is to find
good proofs, where the quality of a proof is described by a measure \m that assigns non-negative rational numbers
to proofs. More precisely, as usual we investigate the complexity of the corresponding decision problem, i.e.,
the problem of deciding whether there is a proof \p with $\m(\p)\leq q$ for a given rational number $q$.
In order to abstract from specific logics and proof calculi, we develop a general framework in which proofs are represented
as labeled, directed hypergraphs, whose hyperedges correspond to single sound derivation steps. To separate the complexity
of generating good proofs from the complexity of reasoning in the underlying logic, we introduce the notion of
a \emph{deriver}, which generates a so-called \emph{derivation structure}. This structure consists of possible proof steps,
from which all proofs of the given consequence can be constructed. Basically, such a derivation structure can be
seen as consisting of all relevant instantiations of the rules of a calculus that can be used to derive the consequence. We restrict
the attention to decidable logics and consider derivers that produce derivation structures of polynomial or exponential size.
Examples of such derivers are consequence-based reasoners for the DLs
$\EL$~\cite{BaBL-IJCAI05,DBLP:journals/jar/KazakovKS14} and
$\ELI$ \cite{DBLP:conf/ijcai/Kazakov09,baader_horrocks_lutz_sattler_2017}, respectively.
In our complexity results, the derivation structure is assumed to be already computed by the deriver,\footnote{%
The highly efficient reasoner ELK~\cite{DBLP:journals/jar/KazakovKS14}
for (an extension of) \EL actually produces a derivation structure, and thus is a deriver in 
our sense.
}
i.e., the complexity of this step is not assumed to be part of the complexity of computing good proofs.
Our complexity results investigate the problem along the following orthogonal dimensions:
we distinguish between
(i)~polynomial and exponential derivers; and
(ii)~whether the threshold value $q$ is encoded in unary or binary.
The obtained complexity upper bounds hold for all instances of a considered setting, whereas the lower bounds
mean that there is an instance (usually based on $\EL$ or $\ELI$) for which this lower bound can be proved.

In our first work in this direction \cite{LPAR23:Finding_Small_Proofs_for}, we focused our attention
on \emph{size} as the measure of proof quality. We could show that the above decision problem
is NP-complete even for polynomial derivers and unary coding of numbers. For exponential derivers, the complexity depends
on the coding of numbers: NP-complete (NExpTime-complete) for unary (binary) coding. For the related measure \emph{tree size} (which assumes that the proof hypergraphs are tree-shaped, \ie cannot reuse already derived consequences), the
complexity turned out to be considerably lower, due to the fact that a Dijkstra-like greedy algorithm can be applied.
\added{
In~\cite{DBLP:conf/dlog/AlrabbaaBBKK20}, we generalized the results by introducing a class of measures called
\emph{$\Psi$-measures}, which contains both size and tree size and for which the same complexity upper bounds as for
size could be shown for polynomial derivers.
We also lifted the better upper bounds for tree size (for polynomial derivers) to \emph{local $\Psi$-measures}, a natural class of proof measures.
In this paper, we extend this line of research by providing a more 
general notion of measures, \emph{monotone recursive \pis-measures}, which now 
also allow to measure the \emph{depth} of a proof. We think that depth 
is an important measure since it measures how much of the proof tree a (human 
or automated) proof checker needs to keep in memory at the same time. We 
analyze these measures not only for polynomial derivers, but this time also consider 
exponential derivers, thus giving insights on how our complexity 
results transfer to more expressive logics. In 
addition to upper bounds for the general class of monotone recursive 
\pis-measures, we show improved bounds for the specific measures considering 
depth 
and tree size, in the latter case improving results 
from~\cite{LPAR23:Finding_Small_Proofs_for}.}
Overall, we thus obtain a comprehensive picture of the complexity landscape for 
the
problem of finding good proofs for DL and other entailments (see 
Table~\ref{table:complexity-results}).

\iftechnicalReport
This is an extended version of the paper \cite{ABB+-CADE21}, including an appendix with more detailed proofs and some auxiliary lemmas.
\fi

%% file: complexity-table.tex
\begin{table}[tb]

\newcommand{\cNewResult}{}%
\newcommand{\cOldResult}{\cellcolor{gray!30}}

\centering
\caption{Overview over existing and new complexity results for deciding the existence of good proofs, 
\wrt polynomial/exponential derivers and unary/binary encoding of the bound~$q$ (known results in gray).}
\label{table:complexity-results}
\begin{tabular}{|l|l|l|l|l|}
    \hline
    Measure
    & $ {\ }^\poly_\un$ &
    $\phantom{\OP}^\poly_\bin$
    &$\phantom{\OP}^\exp_\un$&$\phantom{\OP}^\exp_\bin$\\
    \hline
    \hline
        Size&
        \cOldResult$\NP$~\cite{LPAR23:Finding_Small_Proofs_for}&
        \cOldResult$\NP$~\cite{LPAR23:Finding_Small_Proofs_for}&
        \cOldResult$\NP$~\cite{LPAR23:Finding_Small_Proofs_for}&
        \cOldResult$\NExpTime$~\cite{LPAR23:Finding_Small_Proofs_for}\\
    \hline
    \hline
    Monotone recursive &
    \cNewResult & \cNewResult & \cNewResult & \cNewResult \\
    \pis-measures & 
    \cNewResult\multirow{-2}{*}{$\le \PTime$}&
    \cNewResult\multirow{-2}{*}{$\le \PTime$~\scriptsize{[Th.\ref{th:OP-general-upperBound}]}}&
    \cNewResult\multirow{-2}{*}{$\le \ExpTime$}&
    \cNewResult\multirow{-2}{*}{$\le \ExpTime$~\scriptsize{[Th.\ref{th:OP-general-upperBound}]}} \\
    \hline
    Tree size & 
    \cOldResult$\PTime$~\cite{LPAR23:Finding_Small_Proofs_for}& 
    \cOldResult$\PTime$&
    \cOldResult$\NP$~\cite{LPAR23:Finding_Small_Proofs_for}&
    \cNewResult$\PSpace$~\scriptsize{[Th.\ref{th:OP-tree-exp-binary-PSpace-mem},\ref{th:OP-tree-exp-binary-PSpace-hard}]}
      \\
    \hline
    Depth &
    \cNewResult$\PTime$~\scriptsize{[Th.\ref{th:poly-P-comp}]} & 
    \cNewResult$\PTime$&
    \cNewResult$\PSpace$~\scriptsize{[Th.\ref{th:OP-depth-exp-unary-PSpace}]} &
    \cNewResult$\ExpTime$~\scriptsize{[Th.\ref{th:poly-P-comp}]} \\
    \hline
    Logarithmic depth &
    \cNewResult$\PTime$~\scriptsize{[Cor.\ref{cor:log-depth}]} & 
    \cNewResult$\PTime$&
    \cNewResult$\ExpTime$~\scriptsize{[Cor.\ref{cor:log-depth}]} & 
    \cNewResult$\ExpTime$\\
    \hline
\end{tabular}
\end{table}

%% file: preliminaries.tex
\section{Preliminaries}

Most of our theoretical discussion applies to arbitrary
\emph{logics}~$\Lmc=(\mathcal{S}_\Lmc,\models_\Lmc)$ that consist of a set
$\mathcal{S}_\Lmc$ of
\emph{$\Lmc$-sentences} and a \emph{consequence relation}
\mbox{${\models_\Lmc}\subseteq P(\mathcal{S}_\Lmc)\times \mathcal{S}_\Lmc$} between
\emph{\Lmc-theories}, \ie subsets of
\Lmc-sentences, and
single \Lmc-sentences.
We assume that $\models_\Lmc$ has a semantic definition, \ie for some definition of
\enquote{model}, $\Tmc\models_\Lmc\eta$ holds iff every model of all elements in~\Tmc is
also a model of~$\eta$.
We also assume that the \emph{size}~$|\eta|$ of an \L-sentence~$\eta$ is defined in
some way, \eg by the number of symbols in~$\eta$.
Since \L is usually fixed, we drop the prefix \enquote{\L-} from now on.
For example, \L could be \emph{first-order logic}.
However, we are mainly interested in proofs for DLs, which can be seen as decidable fragments of first-order logic~\cite{baader_horrocks_lutz_sattler_2017}.
In particular, we use specific DLs to show our hardness results.

The syntax of DLs is based on disjoint, countably infinite sets~\NC and~\NR of \emph{concept names}~$A,B,\dots$ and \emph{role names}~$r,s,\dots$, respectively.
Sentences of the DL~$\EL$, called \emph{general concept inclusions (GCIs)}, are of the form $C\sqsubseteq D$, where $C$ and $D$ are \emph{\EL-concepts}, which are built from concept names by applying the constructors $\top$~(\emph{top}), $C\sqcap D$ (\emph{conjunction}), and $\exists r.C$ (\emph{existential restriction} for a role name~$r$).
The DL~\ELI extends~\EL by the role constructor $r^-$ (\emph{inverse role}).
In DLs, finite theories are called \emph{TBoxes} or \emph{ontologies}.

The semantics of DLs is based on first-order interpretations; for details, see~\cite{baader_horrocks_lutz_sattler_2017}.
In Figure~\ref{fig:cr}, we depict a simplified version of the inference rules for \EL from~\cite{DBLP:journals/jar/KazakovKS14}.
For example, $\{A\sqsubseteq\exists r.B,\ B\sqsubseteq C,\ \exists r.C\sqsubseteq D\}\models A\sqsubseteq D$ is a valid inference in \EL.
Deciding consequences in \EL is \PTime-complete~\cite{BaBL-IJCAI05}, and in \ELI %
it is \ExpTime-complete~\cite{BaBL-OWLED08}.
\begin{figure}[tb]
  \input{el-calculus-side-condition}
  \caption{The inference rules for \EL used in \Elk~\cite{DBLP:journals/jar/KazakovKS14}.}
  \label{fig:cr}
\end{figure}

%% file: el-calculus-side-condition.tex
  \centering
  \AXC{\vphantom{I}}
   \LeftLabel{$\mathsf{R}_0$}
  \UIC{$C\sqsubseteq C$}
  \DP
  \quad
  \AXC{\vphantom{I}}
  \LeftLabel{$\mathsf{R}_\top$}
  \UIC{$C\sqsubseteq\top$}
  \DP
  \quad
  \AXC{$C\sqsubseteq D$}
   \LeftLabel{$\mathsf{R}_\sqsubseteq$}
   \RightLabel{$: D\sqsubseteq E \in \Tmc$}
  \UIC{$C\sqsubseteq E$}
  \DP
  \quad
  \AXC{$C\sqsubseteq D \sqcap E$}
   \LeftLabel{$\mathsf{R}_{\sqcap,1}^-$}
  \UIC{$C\sqsubseteq D$}
  \DP
  \\[3ex]
  
    \AXC{$C\sqsubseteq D \sqcap E$}
   \LeftLabel{$\mathsf{R}_{\sqcap,2}^-$}
  \UIC{$C\sqsubseteq E$}
  \DP
  \quad
  \AXC{$C\sqsubseteq D \quad C\sqsubseteq E$}
  \LeftLabel{$\mathsf{R}_\sqcap^+$}
  \UIC{$C\sqsubseteq D \sqcap E$}
  \DP
  \qquad
  \AXC{$C\sqsubseteq\exists r.D \quad D\sqsubseteq E$}
  \LeftLabel{$\mathsf{R}_\exists$}
  \UIC{$C\sqsubseteq\exists r.E$}
  \DP

%% file: proofs-and-derivation-structures.tex
\subsection{Proofs}

We formalize proofs as (labeled, directed) \emph{hypergraphs}
(see Figures~\ref{fig:ex:proof-graph-2},~\ref{fig:ex:proof-tree-2}), which are tuples $(V,E,\ell)$ consisting of 
a finite set~$V$ of \emph{vertices}, a finite set~$E$ of \emph{(hyper)edges} of the form $(S,d)$ with $S\subseteq V$ and $d\in V$, and a
\emph{vertex labeling function} $\el\colon V\to \mathcal{S}_\L$.
Full definitions of such hypergraphs, as well as related notions such as \emph{trees}, \emph{unravelings}, \emph{homomorphisms}, \emph{cycles} can be found in the
\iftechnicalReport
appendix.
\fi
For example, there is a homomorphism from Figure~\ref{fig:ex:proof-tree-2} to Figure~\ref{fig:ex:proof-graph-2}, but not vice versa, and Figure~\ref{fig:ex:proof-tree-2} is the tree unraveling of Figure~\ref{fig:ex:proof-graph-2}.
\begin{figure}[htb]
    \begin{minipage}[b]{0.44\linewidth}
        \centering
        \begin{tikzpicture}[scale=0.7,
        block/.style={
            draw,
            rounded rectangle,
            minimum width={width("$C \sqsubseteq \exists r.(\exists r.A \sqcap B)$")-5pt}},
        he/.style={rounded corners=8pt,->}]
        \path (7,3) node[block, very thick] (x1) {$B \sqsubseteq \exists r.A$}
        (3,3) node[block, very thick] (x2) {$A \sqsubseteq B$}
        (5,2) node[block] (x3) {$A \sqsubseteq \exists r.A$}
        (3,1) node[block] (z) {$A \sqsubseteq B \sqcap \exists r.A$};
        \draw[he] (x1) -- ($(x1)!0.5!(x2)$) -- (x3);
        \draw[he] (x2) -- ($(x1)!0.5!(x2)$) -- (x3);
        \draw[he] (x3.west) -- (x3-|z) -- (z);
        \draw[he] (x2) -- (z);
        \end{tikzpicture}
        \caption{An acyclic hypergraph/proof}\label{fig:ex:proof-graph-2}
    \end{minipage}
    \hfill
    \begin{minipage}[b]{0.53\linewidth}
        \centering
        \begin{tikzpicture}[scale=0.7,
        block/.style={
            draw,
            rounded rectangle,
            minimum width={width("$C \sqsubseteq \exists r.(\exists r.A \sqcap B)$")-5pt}},
        he/.style={rounded corners=8pt,->}]
        \path (7,3) node[block, very thick] (x1) {$B \sqsubseteq \exists r.A$}
        (3,3) node[block, very thick] (x2) {$A \sqsubseteq B$}
        (1,2) node[block, very thick] (y2) {$A \sqsubseteq B$}
        (5,2) node[block] (x3) {$A \sqsubseteq \exists r.A$}
        (3,1) node[block] (z) {$A \sqsubseteq B \sqcap \exists r.A$};
        \draw[he] (x1) -- ($(x1)!0.5!(x2)$) -- (x3);
        \draw[he] (x2) -- ($(x1)!0.5!(x2)$) -- (x3);
        \draw[he] (x3.west) -- (x3-|z) -- (z);
        \draw[he] (y2) -- ($(y2)!0.5!(x3)$) -- (z);
        \end{tikzpicture}
        \caption{A tree hypergraph/proof}\label{fig:ex:proof-tree-2}
    \end{minipage}
\end{figure}

The following definition formalizes basic requirements for hyperedges to be considered valid 
inference steps from a given finite theory.

\begin{definition}[Derivation Structure]\label{def:derivation-structure}
  A \emph{derivation structure} $\ds = (V, E, \el)$ over a finite theory~\Tmc is a hypergraph that is
  \begin{itemize}
  \item \emph{grounded}, \ie
  every leaf $v$ in~\ds is labeled by $\el(v)\in\Tmc$; and
  \item \emph{sound}, \ie for every $(S,d)\in E$, the entailment
$\{\el(s)\mid s\in S\}\models\el(d)$ holds.
\end{itemize}
\end{definition}

We define proofs as special derivation structures that derive a conclusion.

\begin{definition}[Proof]\label{def:proof}
  Given a conclusion~$\eta$ and a finite theory~\Tmc, a \emph{proof for $\Tmc\models\eta$} is a 
  derivation structure $\p = (V, E,\el)$ over~\Tmc such that
  \begin{itemize}
    \item\label{item1:def-proof-non-redundancy} \p contains exactly one
        sink~$v_\eta\in V$, which is labeled by~$\eta$,
    \item \p is acyclic, and
    \item every vertex has at most one incoming edge, \ie there is no vertex $w\in V$ \st there are $(S_1,w),
            (S_2,w)\in E$ with $S_1\neq S_2$.
  \end{itemize}
A \emph{tree proof} is a proof that is a tree.
A \emph{subproof} $S$ of a hypergraph~$H$ is a subgraph of~$H$ that is a proof \st the leaves 
of $S$ are a subset of the leaves of~$H$.
\end{definition}

The hypergraphs in Figures~\ref{fig:ex:proof-graph-2} and~\ref{fig:ex:proof-tree-2} can be seen as proofs in the sense of Definition~\ref{def:proof}, where the sentences of the theory are marked 
with a thick border. Both proofs use the same inference steps, but have different numbers of vertices. They both prove $A\sqsubseteq
B\sqcap\exists r.A$ from $\Tmc=\{ A \sqsubseteq B,\
B \sqsubseteq \exists r.A \}$.
The second proof is a
tree and
the first one a hypergraph without label repetition.
\begin{restatable}{lemma}{finitePathsSinkConclusion}\label{lem:proof-properties}
    Let $\p=(V,E,\el)$ be a proof for~$\Tmc\models\eta$. Then
    \begin{enumerate}
        \item all paths in \p are finite and all longest paths in \p have $v_\eta$ as the target; and
        \item $\Tmc\models \eta$.
    \end{enumerate}
\end{restatable}

Given a proof $\p=(V, E,\el)$ and a vertex $v\in V$, the \emph{subproof of~$\p$ with sink~$v$} is the largest subgraph $\p_v=(V_v,E_v,\el_v)$ of~\p where
  $V_v$ contains all vertices in~$V$ that have a path to $v$ in $\p$.

\input{derivers-short}

%% file: derivers-short.tex
\subsection{\Reasoners}
\label{sec:derivers}

In practice, proofs and derivation structures are constructed by a reasoning system, and in 
theoretical investigations, it is common to define proofs by means of a calculus. To abstract from 
these details, we use the concept of a \emph{deriver} as 
in~\cite{LPAR23:Finding_Small_Proofs_for}, which is a function that, given a theory~\Tmc and 
a conclusion~$\eta$, produces the corresponding derivation structure in which we can look for 
an optimal proof.  However, in practice, it would be inefficient and unnecessary to compute the 
entire derivation structure beforehand when looking for an optimal proof. 
Instead, we allow to access elements in a derivation structure using an oracle, which we can ask 
whether given inferences are a part of the current derivation structure. Similar functionality exists for 
example for the DL reasoner \Elk~\cite{ELK-TRACING}, and may correspond to checking whether 
the inference is an instance of a rule in the calculus. 
Since reasoners may not be complete for proving arbitrary sentences of~\L, we restrict the 
conclusion~$\eta$ to a
subset $C_\L\subseteq \mathcal{S}_\L$ of supported consequences.

\begin{definition}[\Reasoner{}]\label{def:poly-exp-ds}
A \emph{\reasoner}~\R is given by a set $C_\L\subseteq \mathcal{S}_\L$ and a function
that assigns derivation structures to pairs $(\Tmc,\eta)$ of finite theories~$\Tmc\subseteq
\mathcal{S}_\L$ and sentences~$\eta\in C_\L$, such that
$\Tmc\models\eta$ iff $\R(\Tmc,\eta)$ contains a proof for~$\Tmc\models\eta$.
A proof~\p for $\Tmc\models\eta$ is called \emph{admissible \wrt $\R(\Tmc,\eta)$} if there is a homomorphism $h\colon\p\to\R(\Tmc,\eta)$.
We call $\R$ a \emph{polynomial \reasoner} if there exists a polynomial~$p(x)$
such that the size of $\R(\Tmc,\eta)$ is bounded by $p(|\Tmc|+|\eta|)$. \emph{Exponential
\reasoner{}s} are defined similarly by the restriction $|\R(\Tmc,\eta)|\le 2^{p(|\Tmc|+|\eta|)}$.
\end{definition}
\Elk is an example of a polynomial deriver, that is, for a given \EL~theory~\Tmc and \EL sentence $\eta$, $\Elk(\Tmc,\eta)$ contains all allowed instances of the rules shown in Figure~\ref{fig:cr}. As an example for an exponential deriver we use \Eli, which uses the rules from Figure~\ref{fig:ELIcr} and is complete for \ELI theories and conclusions of the form $A\sqsubseteq B$, $A$, $B\in\NC$.
The oracle access for a deriver $\R$ works as follows. Let $\ds=(V,E,\el):=\R(\Tmc,\eta)$ and $V=\{v_1,\dots,v_m\}$. %
    \ds is accessed using the following two functions, 
    where $i,i_1,\dots,i_l$ are indices of vertices and $\alpha$ is a
    sentence:
    \begin{align*}
    [\ds](i_1,\dots,i_l,i) &:=
    \begin{cases}
    \true &\text{if $(\{v_{i_1},\dots,v_{i_{l}}\},v_i)\in E$,} \\
    \false &\text{otherwise;}
    \end{cases} \\
    [\ds](i,\alpha) &:=
    \begin{cases}
    \true & \text{if $\el(v_i)=\alpha$,} \\
    \false & \text{otherwise.}
    \end{cases}
    \end{align*}

\begin{figure}[t]
\input{eli-calculus}
\caption{The inference rules for \ELI~\cite{baader_horrocks_lutz_sattler_2017}. Given a finite 
theory $\Tmc$ in a certain normal form, the rules produce a saturated theory~$\Tmc'$. Here, $K,L,M$ are 
conjunctions of concept names, $A$ is a concept name, $C$ is an \ELI concept of the form $A$, 
$\exists r.M$, or $\forall r. A$, and $r$ is a role name or the inverse of a role name.
  In this calculus conjunctions are implicitly viewed as sets, \ie the order and multiplicity of conjuncts is ignored.}\label{fig:ELIcr}
\end{figure}

In this paper, we focus on polynomial and exponential derivers, for which we further make the following technical assumptions: 1) $\R(\Tmc,\eta)$ does not contain two vertices with the same label; 2) the number of premises in an inference is polynomially bounded by $|\Tmc|$ and $|\eta|$; and 3) the size of each label is polynomially bounded by $|\Tmc|$ and $|\eta|$. While~1) is without loss of generality, 2) and~3) are not. If a deriver does not satisfy~2), we may be able to fix this by splitting inference steps. Assumption~3) would not work for derivers with higher complexity, but is required in our setting to avoid trivial complexity results for exponential derivers. We furthermore assume that for polynomial and exponential derivers, the polynomial $p$ from Definition~\ref{def:poly-exp-ds} bounding the size of derivation structures is known.

%% file: eli-calculus.tex
\centering
  \AXC{}
 \LeftLabel{$\mathsf{CR1}$}
   \RightLabel{ if $A\in K$ and $K$ appears in $\Tmc'$}
\UIC{$K\sqsubseteq A$}
\DP
\\[3ex]
   \AXC{$M \sqsubseteq A$ for all $A\in K$, $K\sqsubseteq C$}
 \LeftLabel{$\mathsf{CR2}$}
   \RightLabel{ if $M$ appears in $\Tmc'$}
\UIC{$M\sqsubseteq C$}
\DP
\\[3ex]
\AXC{$M\sqsubseteq \exists r.L$ \quad $L\sqsubseteq \forall r^-.A$}
 \LeftLabel{$\mathsf{CR3}$}
\UIC{$M\sqsubseteq A$}
\DP
\qquad
\AXC{$L\sqsubseteq \exists r.M$ \quad $L\sqsubseteq \forall r.A$}
\LeftLabel{$\mathsf{CR4}$}
\UIC{$L\sqsubseteq\exists r.(M\sqcap A)$}
\DP

%% file: measuring-proofs.tex
\section{Measuring Proofs}
\label{sec:measures}

To formally study quality measures for proofs, we developed the following definition, which will be instantiated with concrete measures later.
Our goal is to find proofs that minimize these measures, \ie lower numbers are better.

\begin{definition}[\pis-Measure]\label{def:measure}
	A \emph{(quality) measure} is a function $\m\colon \mathrm{P}_{\L} \rightarrow
	\Rationals_{\ge 0}$, where $\mathrm{P}_{\L}$ is the set of all proofs over~\L and
	$\Rationals_{\ge 0}$ is the set of non-negative rational numbers.
	We call \m a \emph{\pis-measure} if, for every $\p\in \mathrm{P}_{\L}$, the following hold.
	\begin{enumerate}[leftmargin=2.2em]
		\item[\pol] $\m(\p)$ is computable in \textbf{\underline{p}}olynomial time in the size
		of~$\p$.
		\item[\si] Let $h\colon \p\to H$ be any homomorphism, and $\p'$ be any
		subproof of the \textbf{\underline{h}}omomorphic \textbf{\underline{i}}mage~$h(\p)$ that is minimal (\wrt \m) among all such subproofs having the same sink. Then
		$\m(\p')\le\m(\p)$.
	\end{enumerate}
\end{definition}
Intuitively, a \pis-measure~\m does not increase when the proof gets smaller, either when parts
of the proof are removed (to obtain a subproof) or when parts are merged (in a homomorphic
image).
For example, $\msize((V,E,\ell)):=|V|$ is a \pis-measure, called the \emph{size} of a proof, and we have already investigated the complexity of the following deicision problem for~$\msize$ in~\cite{LPAR23:Finding_Small_Proofs_for}.
\begin{definition}[Optimal Proof]
	\label{def:best}
	Let \R be a \reasoner and \m be a measure.
	Given a finite theory~\Tmc and a sentence~$\eta\in C_\L$ \st $\Tmc \models \eta$, an
	admissible proof \p \wrt
	$\R(\Tmc,\eta)$ is called \emph{optimal}
	\wrt \m if $\m(\p)$ is minimal among all such proofs.
	The associated decision problem, denoted $\OP(\R,\m)$, is to decide, given \Tmc
	and~$\eta$ as above and $q\in\Rationals_{\ge 0}$, whether there is an admissible proof~\p \wrt
	$\R(\Tmc,\eta)$ with $\m(\p)\le q$.
\end{definition}
For our complexity analysis, we distinguish the encoding of~$q$ with a subscript (\un/\bin), \eg $\OP_\un(\R,\m)$.

We first show that if \p is optimal \wrt a \pis-measure~\m and $\R(\Tmc,\eta)$,
then the homomorphic image of~\p in $\R(\Tmc,\eta)$ is also a proof.
Thus, to decide $\OP(\R,\m)$ we can restrict our search to proofs that are subgraphs of
$\R(\Tmc,\eta)$.
\begin{restatable}{lemma}{LemAdmHyperproofInside}\label{lem:adm-hyperproof-inside}
	For any deriver~\R and \pis-measure \m, if there is an admissible proof \p \wrt $\R(\Tmc,\eta)$ with
	$\m(\p)\le q$ for some $q\in \Rationals_{\ge 0}$, then there exists a subproof $\q$ of
	$\R(\Tmc,\eta)$ for $\Tmc\models\eta$ with $\m(\q)\le q$.
\end{restatable}
In particular, this shows that an optimal proof always exists.%

\begin{restatable}{corollary}{CorExistence}\label{cor:existence}
	For any \reasoner~\R and \pis-measure~\m, if $\Tmc\models\eta$, then there is an
	optimal proof
	for $\Tmc\models\eta$ \wrt \R and~\m.
\end{restatable}
\begin{proof}
	By Definition~\ref{def:poly-exp-ds}, the derivation structure $\R(\Tmc,\eta)$ contains at
	least one proof for~$\Tmc\models\eta$. Since $\R(\Tmc,\eta)$ is finite, there are finitely
	many proofs for $\Tmc\models\eta$ contained in $\R(\Tmc,\eta)$. The finite set of all
	\m-weights of these proofs always has a minimum. Finally, if there were an admissible proof
	weighing less than this minimum, it would contradict
	Lemma~\ref{lem:adm-hyperproof-inside}.
	\qed
\end{proof}

%% file: monotone-recursive-measures.tex
\subsection{Monotone Recursive Measures}\label{sec:trees}

Since the complexity of $\OP(\R,\m)$ for \pis-measures in general is quite high~\cite{LPAR23:Finding_Small_Proofs_for}, in this paper we focus on a subclass of measures that can be evaluated recursively.

\begin{definition}
\label{def:recursive}
  A \pis-measure~\m is \emph{recursive} if there exist
  \begin{itemize}
    \item a \emph{leaf function} $\leaf_\m\colon\mathcal{S}_\Lmc\to\mathbb{Q}_{\ge 0}$ and
    \item a partial \emph{edge function~$\edge_\m$}, which maps (i)~the labels $(\Smc,\alpha)$ of a hyperedge and (ii)~a finite multiset~\q of already computed intermediate weights in~$\mathbb{Q}_{\ge 0}$ to a combined weight $\edge_\m\big((\Smc,\alpha),\q\big)$
  \end{itemize}
    such that, for any proof $\p=(V,E,\ell)$ with sink~$v$, we have
    \[
    \m(\p)=\begin{cases}
        \leaf_\m(\ell(v))
        & \text{if $V=\{v\}$,}\\
        \edge_\m\big(\ell(S,v),\{\m(\p_w)\mid w\in S\}\big)
        & \text{if $(S,v)\in E$.}
    \end{cases}
    \]

    Such a measure is \emph{monotone} if, for any multiset $\q$, whenever $q\in\q$ and $\q'=(\q\setminus\{q\})\cup\{q'\}$ with $q'\le q$ and both $\edge_\m\big((\Smc,\alpha),\q'\big)$ and $\edge_\m\big((\Smc,\alpha),\q\big)$ are defined, then $\edge_\m\big((\Smc,\alpha),\q'\big)\le\edge_\m\big((\Smc,\alpha),\q\big)$.
\end{definition}
Intuitively, a recursive measure~\m can be computed in a bottom-up fashion starting with the 
weights of the leaves given by~$\leaf_\m$. The function~$\edge_\m$ is used to recursively 
combine the weights of the direct subproofs into a weight for the full proof.
This function is well-defined since in a proof every vertex has at most one incoming edge.
We require~$\edge_\m$ to be defined only for inputs $\big((\Smc,\alpha),\q\big)$ that actually correspond to a valid proof in~\L, \ie where $\Smc\models_\L\alpha$ and \q consists of the weights of some proofs for the sentences in~\Smc.
For example, if~\m always yields natural numbers, we obviously do not need $\edge_\m$ to be defined for multisets containing fractional numbers.

In this paper, we are particularly interested in the following monotone recursive \pis-measures.
\begin{itemize}
    \item The \emph{depth} \mdepth of a proof is defined by
    \[
    \leaf_{\mdepth}(\alpha):=0 \text{ and }
    \edge_{\mdepth}\big((\Smc,\alpha),\q\big):=1+\max\q.
    \]
    \item The \emph{tree size} \mtree is given by
    \[
    \leaf_{\mtree}(\alpha):=1 \text{ and }
    \edge_{\mtree}\big((\Smc,\alpha),\q\big):=1+\sum\q.
    \]
\end{itemize}

What distinguishes \emph{tree size} from \emph{size} is that vertices are counted multiple times if they are used in several subproofs. The name \emph{tree size} is inspired by the fact that it can be interpreted as the \emph{size} of the tree unraveling of a given proof (\cf Figures~\ref{fig:ex:proof-graph-2} and~\ref{fig:ex:proof-tree-2}).
In fact, we show in the
\iftechnicalReport
appendix
\fi
that all recursive \pis-measures are invariant under unraveling.
This indicates that \emph{tree size}, \emph{depth} and other monotone recursive \pis-measures are especially well-suited
for cases where proofs are presented to users in the form of trees. This is for example the case for
the proof plugin for \Protege~\cite{KaKS-DL17}.

\begin{restatable}{lemma}{LemTableIsPSI}\label{lem:table-psi}
    \emph{Depth} and \emph{tree size} are monotone recursive \pis-measures.
\end{restatable}

%% file: complexity-results.tex
\section{Complexity Results}

We investigate the decision problem \OP for monotone recursive \pis-measures. We first show upper bounds for the general case, and then consider measures for \emph{depth} and \emph{tree size}, for which we obtain even lower bounds. An artificial modification of the \emph{depth} measure gives a lower bound matching the general upper bound even if unary encoding is used for the threshold~$q$.

\input{complexity-results-general}
\input{complexity-results-depth}

\input{complexity-results-tree}

%% file: complexity-results-general.tex
\subsection{The General Case}

\begin{algorithm}[tb]
    \KwIn{A derivation structure~$\R(\Tmc,\eta)=(V,E,\ell)$, a monotone recursive \pis-measure~\m}
    \KwOut{An optimal proof of $\Tmc\models\eta$ \wrt $\R(\Tmc,\eta)$ and~\m}

    \BlankLine
    $Q:=\emptyset$\;
    \lForEach{$e\in E$}{$k(e):=0$}
    \ForEach{$v\in V$}{
        \If{$\ell(v)\in\Tmc$}{
            $\p(v):=(\{v\},\emptyset,\ell|_{\{v\}})$; $Q:=Q\cup\{v\}$
            \label{l:tbox}
            \tcp*[r]{$\ell(v)$ is in the theory}
        }
        \ElseIf{$(\emptyset,v)\in E$}{
            $\p(v):=(\{v\},\{(\emptyset,v)\},\ell|_{\{v\}})$; $Q:=Q\cup\{v\}$
            \label{l:tautology}
            \tcp*[r]{$\ell(v)$ is a tautology}
        }
        \Else{$\p(v):=\text{undefined}$}
    }
    \While{$Q\neq\emptyset$}{
        choose $v\in Q$ with minimal $\m(\p(v))$
        \label{l:choose}
        \tcp*[r]{$\p(v)$ is optimal for~$\ell(v)$}
        $Q:=Q\setminus\{v\}$\; \label{l:remove-from-q}
        \ForEach{$e=(S,d)\in E$ with $v\in S$}{
            $k(e):=k(e)+1$\;
            \If(\tcp*[f]{all source vertices have been reached}){$k(e)=|S|$\label{l:counter}}{
                $\p:=(S\cup\{d\},e,\ell_{S\cup\{d\}})\cup\bigcup_{s\in S}\p(s)$ \label{l:new-proof}
                \tcp*[r]{construct new proof}
                \If{$\p$ is acyclic\label{l:acyc}}{
                    \If{$\p(d)$ is undefined or $\m(\p(d))>\m(\p)$\label{l:better}}{
                        $\p(d):=\p$; $Q:=Q\cup\{d\}$ \label{l:update}
                        \tcp*[r]{$\p$ is better for~$\ell(d)$}
                    }
                }
            }
        }
    }
    \KwRet{$\p(v_\eta)$, where $\ell(v_\eta)=\eta$}\label{l:return}
    \caption{A Dijkstra-like algorithm\label{alg:dijkstra}}
\end{algorithm}
Algorithm~\ref{alg:dijkstra} describes a Dijkstra-like approach that is inspired by the algorithm
in~\cite{DBLP:journals/dam/GalloLP93} for finding minimal hyperpaths \wrt so-called
\emph{additive weighting functions}, which represent a subclass of monotone recursive \pis-measures.
The algorithm progressively discovers proofs $\p(v)$ for~$\ell(v)$ that are contained in $\R(\Tmc,\eta)$.
If it reaches a new vertex~$v$ in this process, this vertex is added to the set~$Q$.
In each step, a vertex with minimal weight $\m(\p(v))$ is chosen and removed from~$Q$.
For each hyperedge~$e=(S,d)\in E$, a counter~$k(e)$ is maintained that is increased whenever a vertex $v\in S$ is chosen.
Once this counter reaches $|S|$, we know that all source vertices of~$e$ have been processed.
The algorithm then constructs a new proof~\p for~$\ell(d)$ by joining the proofs for the source vertices using the current hyperedge~$e$.
This proof~\p is then compared to the best previously known proof~$\p(d)$ for~$\ell(d)$ and $\p(d)$ is updated accordingly.
For Line~\ref{l:return}, recall that we assumed $\R(\Tmc,\eta)$ to contain no two vertices with
the same label, and hence it contains a unique vertex~$v_\eta$ with label~$\eta$.

\begin{restatable}{lemma}{LemDijkstra}\label{lem:dijkstra}
    For any monotone recursive \pis-measure~\m and \reasoner~\R, Algorithm~\ref{alg:dijkstra}
    computes an optimal proof in time polynomial in the size of~$\R(\Tmc,\eta)$.
\end{restatable}
Since we can actually compute an optimal proof in polynomial time in the size of the whole derivation structure, it is irrelevant how the upper bound~$q$ in the decision problem~\OP is encoded, and hence the following results follow.

\begin{theorem}\label{th:OP-general-upperBound}
    For any monotone recursive \pis-measure~\m and polynomial deriver~\R, $\OP_\bin(\R,\m)$ is
    in~\PTime. It is in \ExpTime for all exponential derivers~\R.
\end{theorem}

%% file: complexity-results-depth.tex
\subsection{Proof Depth}

We now consider the measure \mdepth in more detail. We can show lower bounds of~\PTime 
and \ExpTime for polynomial and exponential derivers, respectively, although the latter only 
holds for upper bounds~$q$ encoded in binary.

Since our definition of $\OP(\R,\m)$ requires that the input entailment $\Tmc\models\eta$ already holds, we cannot use a straightforward reduction from the entailment problem in \EL or \ELI, however.
Instead, we show that ordinary proofs~\p for $\Tmc\models\eta$ satisfy $\m(\p)\le q$ for some~$q$, and then extend the TBox to~$\Tmc'$ in order to create an artificial proof~$\p'$ with $\m(\p')>q$.
In this way, we ensure that $\Tmc'\models\eta$ holds and can use~$q$ to distinguish the 
artificial from the original proofs.

For \ELI, we can use an observation from~\cite[Example~6.29]{baader_horrocks_lutz_sattler_2017} for this purpose.

\begin{restatable}[\!\!\cite{baader_horrocks_lutz_sattler_2017}]{proposition}{PropELILargeProof}
  \label{prop:eli-large-proof}
    For every $q\in\mathbb{Q}_{\ge 0}$ and $\ELI$ sentence of the form $A\sqsubseteq B$, where $A,B\in\NC$, one can construct in time polynomial in~$q$ an $\ELI$ theory~$\Tmc$ such $\Tmc\models A\sqsubseteq B$, and every 
    proof for $\Tmc\models A\sqsubseteq B$ in $\Eli$ is of depth larger than~$2^q$.
\end{restatable}

We can now reduce the entailment problems for \EL and \ELI to obtain the claimed lower bounds.

\begin{restatable}{theorem}{multiHardnessTheorem}\label{th:poly-P-comp}
	The problems $\OP_\un(\Elk,\mdepth)$ and $\OP_\bin(\Eli,\mdepth)$ are \PTime-hard and \ExpTime-hard, respectively.
\end{restatable}
\begin{proof}
    For the \PTime-hardness, we provide a \LogSpace-reduction from the entailment problem of a GCI
    $A\sqsubseteq B$ with two concept names~$A,B$ from an \EL-theory~\Tmc, which is
    \PTime-hard~\cite{baader_horrocks_lutz_sattler_2017}.
    To reduce this problem to $\OP_\un(\Elk,\mtree)$, we need to find a theory~$\Tmc'$ and a
    number~$q$ such that $\Tmc'\models A\sqsubseteq B$ holds, and moreover $\Tmc\models A\sqsubseteq B$ holds
    iff $\Elk(\Tmc',A\sqsubseteq B)$ contains a proof of $\Tmc'\models A\sqsubseteq B$ of depth
    $\le q$ (\cf Lemma~\ref{lem:adm-hyperproof-inside}).
    
    First, observe that, since proofs must be acyclic, the depth of any proof of $A\sqsubseteq B$
    from~\Tmc is bounded by $q:=|\Elk(\Tmc,A\sqsubseteq B)|$, whose size in unary encoding is
    polynomial in the size of~\Tmc.
    We now construct
    \[ \Tmc' := \Tmc \cup \{ A\sqsubseteq A_1,\ A_1\sqsubseteq A_2, \dots, A_{q+2}\sqsubseteq B\}, \]
    \iftechnicalReport
    where $A_1,\dots,A_q$ are concept names that do not occur in~\Tmc.
    \fi
    Furthermore, the existence of an admissible proof for $\Tmc'\models A\sqsubseteq B$ of depth at
    most~$q$ is equivalent to $\Tmc\models A\sqsubseteq B$, since any proof that uses the new
    concept names must take $q+1$ consecutive steps using rule~$\mathsf{R}_\sqsubseteq$, \ie must be
    of depth~$q+1$.
    Moreover, we can compute $q$ (in binary representation) and output it in unary representation
    using a logarithmically space-bounded Turing machine, and similarly for~$\Tmc'$.
    Hence, the above construction constitutes the desired \LogSpace-reduction.
    
    For the remaining result, we can use similar arguments about the exponential deriver~\Eli, where
    entailment is \ExpTime-hard \cite{baader_horrocks_lutz_sattler_2017}:
    \begin{itemize}
        \item the minimal depth of a proof in an exponential derivation structure is at most
        exponential, and this exponential bound~$q$ can be computed in polynomial time using binary
        encoding;
        \item by Proposition~\ref{prop:eli-large-proof}, there is an \ELI theory~\Tmc of size polynomial in the size of the binary encoding
        of~$q$ such that $\Tmc\models A\sqsubseteq B$ and any proof for $\Tmc\models A\sqsubseteq B$ must have at least depth~$q+1$.
        \qed
    \end{itemize}
\end{proof}

To demonstrate that the generic upper bounds from Theorem~\ref{th:OP-general-upperBound} are tight
even for unary encoding, we quickly consider the artificial measure \mldepth (\emph{logarithmic
depth}), which simply computes the (binary) logarithm of the depth of a given proof.
This is also a monotone recursive \pis-measure, since the logarithmic depth contains exactly the
same information as the depth itself.
It is easy to obtain the following lower bounds from the previous results about \mdepth.

\begin{corollary}\label{cor:log-depth}
  $\OP_\un(\Elk,\mldepth)$ is \PTime-hard and $\OP_\un(\Eli,\mldepth)$ is \ExpTime-hard.
\end{corollary}
\begin{proof}
  \iftechnicalReport
  Regardless of the chosen deriver~\R, $\OP_\bin(\R,\mdepth)$ can be \LogSpace-reduced to $\OP_\un(\R,\mldepth)$, because in order to find a proof of depth at most $q$ (with $q$ given in binary), one can equivalently look for a proof whose logarithmic depth is bounded by the value $\log q$. 
  \fi
  The unary encoding of~$\log q$ has the same size as the binary encoding of~$q$ and can be computed in \LogSpace by flipping all bits of the binary encoding of~$q$ to~$1$.
  \qed
\end{proof}

We now return to \mdepth and cover the remaining case of exponential derivers and unary encoding of the 
upper bound~$q$.

\begin{restatable}{theorem}{depthPSpace}\label{th:OP-depth-exp-unary-PSpace}
  $\OP_\un(\R,\mdepth)$ is in \PSpace for any exponential deriver $\R$. It is \PSpace-hard for the exponential deriver $\R=\Eli$.
\end{restatable}
\begin{proof}
    For the upper bound, we employ a depth-first guessing strategy: we guess a proof of depth
    at most~$q$, where at each time point we only keep one branch of the proof in memory. As the
    length of this branch is bounded by~$q$, and due to our assumptions on derivers, this procedure
    only requires polynomial space.
    
    For the lower bound, we provide a reduction from the \PSpace-complete QBF problem (satisfiability of quantified Boolean formulas). Let $\quant_1 
    x_1\quant_2
    x_2\ldots\quant_m x_m.\phi$ be a quantified Boolean formula, where for $i\in\{1,\ldots,m\}$,
    $\quant_i\in\{{\exists},{\forall}\}$, and $\phi$ is a formula over $\{x_1,\ldots,x_m\}$. 
    We 
    assume
    $\phi$ to be in negation normal form, that is, negation only occurs directly in front of a 
    variable.
    We
    construct an \ELI theory $\Tmc$ and a number $q$, both of size polynomial in the size of the
    formula, such that $\Tmc\models A\sqsubseteq B$ holds (\cf Definition~\ref{def:best}) and
    $\Tmc$ has a proof for $A\sqsubseteq B$ of depth $q$ iff the QBF formula is valid.
    We use
    two roles $r_1$, $r_2$ to deal with the variable valuations, concept names $A_0$, 
    $\ldots$,
    $A_{m}$ to count the quantifier nesting, and a concept name $A_\psi$ for every 
    sub-formula
    $\psi$ of $\phi$. In addition, we use the concept names $A$ and $B$ occurring in
    the conclusion, and two concept names $B_1$ and $B_2$.
    
    The concept name $A$ initializes 
    the 
    formula at
    quantifier nesting level $0$:
    \begin{align*}
    A\sqsubseteq A_0
    \end{align*}
    For every $i\in\{1,\ldots,m\}$, $\Tmc$ contains the following sentence to select a
    truth valuation for $x_i$, increasing the nesting depth in each step.
    \begin{align}
    A_{i-1}&\sqsubseteq \exists r_1.(A_{i}\sqcap A_{x_i})\\
    A_{i-1}&\sqsubseteq \exists r_2.(A_{i}\sqcap A_{\neg x_i}).
    \end{align}
    To ensure truth valuations are kept along the role-successors, we use the following 
    sentences for
    every $l\in\{x_i,\neg x_i\mid 1\leq i\leq m\}$:
    \begin{align}
    A_l&\sqsubseteq\forall r_1.A_l \qquad A_l\sqsubseteq\forall r_2.A_l
    \end{align}
    The following GCIs are now used to evaluate $\phi$. For every conjunction 
    $\psi=\psi_1\wedge\psi_2$
    occurring in $\phi$, we use:
    \begin{align}
    A_{\psi_1}\sqcap A_{\psi_2}\sqsubseteq A_\psi,
    \end{align}
    and for every disjunction $\psi=\psi_1\vee\psi_2$, we use:
    \begin{align}
    A_{\psi_1}\sqsubseteq A_\psi \qquad A_{\psi_2}\sqsubseteq A_\psi
    \end{align}
    Finally, the following GCIs are used to propagate the result of the evaluation back 
    towards the
    start.
    \begin{align}
    A_\phi&\sqsubseteq B \\
    A_i\sqcap B&\sqsubseteq\forall r_1^-.B\quad &
    A_i\sqcap B&\sqsubseteq\forall r_2^-.B  &&\quad\text{ if }\quant_i={\exists}\\
    A_i\sqcap B&\sqsubseteq\forall r_1^-.B_1\qquad &
    A_i\sqcap B&\sqsubseteq\forall r_2^-.B_2 \qquad
    B_1\sqcap B_2\sqsubseteq B &&\quad\text{ if }\quant_i={\forall}
    \end{align}
    One can now show that there exists a proof for $A\sqsubseteq B$ from $\Tmc$ of depth at most~$q$
    iff the QBF formula is valid, where $q$ is polynomial and determined by the size and structure
    of~$\phi$.
    Finally, we can extend \Tmc with the sentences from Proposition~\ref{prop:eli-large-proof} to ensure that $\Tmc\models A\sqsubseteq B$ holds while retaining this equivalence.
    \qed
\end{proof}

%% file: complexity-results-tree.tex
\subsection{The Tree Size Measure}

The tree size measure was discussed already in~\cite{LPAR23:Finding_Small_Proofs_for}, where tight bounds were provided for polynomial derivers and exponential derivers with unary encoding. For the case of exponential derivers with binary encoding, only an \ExpTime upper bound was provided, and the precise complexity left open. We improve this result by showing that $\OP_\bin(\R,\mtree)$ can indeed be decided in \PSpace.

\begin{restatable}{theorem}{treePSapceUpperbound}\label{th:OP-tree-exp-binary-PSpace-mem}
    For any exponential deriver \R, $\OP_\bin(\R,\mtree)$ is in $\PSpace$.
\end{restatable}
\begin{proof}[sketch]
    We describe a non-deterministic procedure for $\OP_\bin(\R,\mtree)$, 
    in polynomial space.
    Let \Tmc be a theory, $\eta$ the goal sentence, and $q$ a rational number in binary 
    encoding.
    By Lemma~\ref{lem:adm-hyperproof-inside}, it suffices to find a proof \p for 
    $\Tmc\models\eta$ in $\R(\Tmc,\eta)$ with $\mtree(\p)\le q$.
    The
    procedure guesses such a proof starting from the conclusion, while keeping in memory a set $S$ of tuples $\tup{\eta',q'}$, where $\eta'$ is a sentence and 
    $q'\leq q$ a
    rational number. Intuitively, such a tuple states: \enquote{We still need to guess a proof for $\eta'$ of tree size at most $q'$.}
    \begin{enumerate}
        \item Initialize $S:=\{\tup{\eta,q}\}$.
        \item While $S\neq\emptyset$,
        \begin{enumerate}
            \item select from $S$ a tuple $\tup{\eta',q'}$ such that for all tuples 
            $\tup{\eta'',q''}\in S$ it holds that $q''\geq q'$;
            \item guess a hyperedge $(\{v_1,\dots,v_m\},v')$ in $\R(\Tmc,\eta)$ (using the oracle access described in Section~\ref{sec:derivers}) and $m$  
            numbers $q_1$, $\ldots$, $q_m$, such that $\ell(v')=\eta'$ and 
            $q_1+\ldots+q_m+1\le q'$;
            and
            \item replace $\tup{\eta',q'}$ in~$S$ by the tuples $\tup{\el(v_1),q_1}$, $\ldots$, 
            $\tup{\ell(v_m),q_m}$.
        \end{enumerate}
    \end{enumerate}
    \input{illustration-tree-algorithm}
    There is a proof for $\Tmc\models\eta$ of tree size at most $q$ iff every step in the 
    algorithm is successful. To show that it only requires polynomial space, we show that 
    during the computation, the number of elements in $S$ is always polynomially bounded. 
    For this, we show 
    that the elements in $S$ can 
    always be organized into a tree with the following properties:
    \begin{enumerate}[label=\textbf{S\arabic*}]
        \item\label{pspace-root} the root is labeled with $\epsilon$,
        \item\label{pspace-labels} every other node is labeled with a distinct element from 
        $S$,
        \item\label{pspace-two} every node that is not the root or a leaf has at least 2 
        children,
        \item\label{pspace-p} every node has at most $p$ children, where $p$ is the maximal number of premises in any inference in $\R(\Tmc,\eta)$, 
        which we assumed to be polynomial in the input,
        \item\label{pspace-half} every node $\tup{\eta',q'}$ has at most 1 child 
        $\tup{\eta'',q''}$ 
        that is not a leaf and for this child it holds that $q''< \frac{q'}{2}$,
        \item\label{pspace-sum} for every node labeled $\tup{\eta',q'}$ with children 
        labeled 
        $\tup{\eta_1,q_1}$, $\ldots$, $\tup{\eta_m,q_m}$, we have $q_1+\ldots+q_m<q'$.
    \end{enumerate}
    We prove this by induction on the steps of the algorithm, where in each step, we either replace 
    one tuple in the tree, or put the new tuples 
    under the leaf with the currently smallest value~(see Fig.\ref{fig:illustration-tree-algorithm}).
    By~\ref{pspace-two} and because every number in $S$ is bounded by~$q$, we can show that the tree has depth at most $\log_2{q}$, which with~\ref{pspace-p} and~\ref{pspace-half} implies that it has at most $p\cdot \log_2 q$ nodes. \ref{pspace-labels} then implies that
    that
    $\left|S\right|\leq p\cdot\log_2{q}$ is always satisfied, and thus that $S$ is polynomially bounded.\qed
\end{proof}

A corresponding lower bound can be found for the exponential deriver $\Eli$ by a reduction of the word problem for deterministic Turing machines with polynomial space bound.

\begin{restatable}{theorem}{treePSpaceLowerBound}\label{th:OP-tree-exp-binary-PSpace-hard}
    For the exponential deriver $\Eli$, $\OP_\bin(\Eli,\mtree)$ is $\PSpace$-hard.
\end{restatable}
\begin{proof}[sketch]
    Let $T=\tup{Q,\Gamma,\blank,\Sigma, \delta,q_0,F}$ be a deterministic Turing machine,
    where $Q$ is
    the set of states, $\Gamma$ the tape alphabet, $\blank\in\Gamma$ the
    blank symbol, $\Sigma\subseteq\Gamma$ the input alphabet,
    $\delta:Q\times\Gamma\not\rightarrow
    Q\times\Gamma\times\{-1,0,+1\}$ the partial transition function,
    $q_0$ the initial state, and
    $F\subseteq
    Q$ the accepting states. We assume that $T$ is polynomially space bounded, that is, there is a polynomial $p$ such 
    that on
    input
    words $w\in\Sigma^*$, $T$ only accesses the first $p(\lvert w\rvert)$ cells of the 
    tape. For
    a
    word~$w$, we denote by $w[i]$ its $i$th letter.
    For some fixed word $w$, we construct a theory $\Tmc$ using the following names, where $k=p(\lvert w\rvert)$:
    \begin{itemize}
        \item \exFont{Start} marks the inital and \exFont{Accept} an accepting configuration;
        \item to denote that we are in state $q\in Q$, we use a concept name $S_q$;
        \item for every $a\in\Gamma$ and $i\in\{0,\ldots,k\}$, we use a concept name 
        $A_i^a$
        denoting that the letter $a$ is on tape position $i$;
        \item for every $i\in\{0,\ldots,k\}$, we use the concept name $P_i^{+}$ to denote
        that the head is currently on position~$i$, and  $P_i^{-}$ to denote that it is not;
        \item the role $r$ is used to express the transitions between the configurations.
    \end{itemize}
    For convenience, we present the theory not in the required normal form, but 
    aggregate
    conjunctions on the right.
    The following sentence describes the initial configuration.
    \begin{align}
    \exFont{Start}\sqsubseteq S_{q_0}
    \sqcap\bigsqcap_{i=0}^{\lvert w\rvert-1}A_i^{w[i]}
    \sqcap\bigsqcap_{i=\lvert w\rvert}^{k}A_i^{\blank}
    \sqcap P_0^{+}\sqcap\bigsqcap_{i=1}^{k}P_i^{-}
    \label{al:initial}
    \end{align}
    The transition from one configuration to the next is encoded with the following sentences 
    for every $i\in\{0,\dots,k\}$ and every
    $\tup{q,a}\in Q\times\Gamma$ with $\delta(q,a)=\tup{q',b,d}$:
    \begin{align}
    S_q\sqcap A_i^a\sqcap P_i^+ &\sqsubseteq \exists r.S_{q'}\sqcap \forall 
    r.A_i^b\sqcap
    \forall r.P_{i+d}^+\sqcap\bigsqcap_{j\in\{0,\ldots,k\}\setminus\{i+d\}}\forall r.P_j^-
    \label{al:transition-function}
    \\
    A_i^a\sqcap P_i^-&\sqsubseteq\forall r.A_i^a
    \label{al:keep-tape}
    \end{align}
    Finally, we use the following sentences to detect accepting configurations and propagate 
    the
    information of acceptance back to the initial configuration
    \begin{align}
    S_f & \sqsubseteq \exFont{Accept}\text{ for all $f\in F$,}
    \label{al:final}
    \\
    \label{al:accept}\exFont{Accept}&\sqsubseteq\forall r^-.\exFont{Accept}
    \end{align}
    One can find a number $q$ exponential in $k$ and the size of $T$ s.t. that there is a proof 
    for $\Tmc\models\exFont{Start}\sqsubseteq\exFont{Accept}$ with tree size at most $q$ iff $T$ 
    accepts $w$. Using Proposition~\ref{prop:eli-large-proof}, we can extend $\Tmc$ to a theory 
    $\Tmc'$ s.t. $\Tmc'\models\exFont{Start}\sqsubseteq\exFont{Accept}$, while a proof of tree 
    size $q$ exists iff $T$ accepts $w$ (observe that $\mtree(\p)\ge\mdepth(\p)$ holds for all proofs~\p).\qed
\end{proof}

%% file: illustration-tree-algorithm.tex
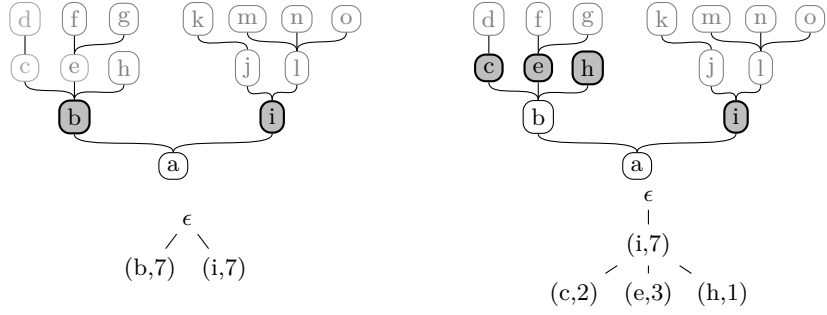
\begin{figure}[tb]
  \newcommand{\hiddenColor}{gray!70}
        \begin{minipage}{.5\linewidth}
            \centering
        \begin{tikzpicture}[
            sibling distance=2em,
            level distance= 2em,
            every node/.style = {shape=rectangle, rounded corners,
              draw, align=center},
              edge from parent path={(\tikzparentnode.north) .. 
                    controls +(0,.225) and +(0,-.225).. (\tikzchildnode.south)}
              ]
            \node {a} [grow'=up]
              child { node[thick,fill=gray!50] {b} 
                child { node[draw=\hiddenColor] {\color{\hiddenColor}c} 
                  child { node[draw=\hiddenColor] {\color{\hiddenColor}d}}
                }
                child { node[draw=gray!70] {\color{gray}e}
                  child[missing]
                  child { node[draw=gray] {\color{gray}f}}
                  child { node[draw=gray] {\color{gray}g}}
                }
                child { node[draw=gray] {\color{gray}h}
                }
              }
              child[missing]
              child[missing]
              child[missing]
              child { node[thick,fill=gray!50] {i}
                child { node[draw=gray] {\color{gray}j}
                  child { node[draw=gray] {\color{gray}k} }
                  child[missing]
                  child[missing]
                }
                child { node[draw=gray] {\color{gray}l} 
                  child { node[draw=gray] {\color{gray}m}}
                  child { node[draw=gray] {\color{gray}n}}
                  child { node[draw=gray] {\color{gray}o}}
                } 
              };
          \end{tikzpicture}
        \end{minipage}%
        \begin{minipage}{.5\linewidth}
            \centering
        \begin{tikzpicture}[
            sibling distance=2em,
            level distance= 2em,
            every node/.style = {shape=rectangle, rounded corners,
              draw, align=center},
              edge from parent path={(\tikzparentnode.north) .. 
                    controls +(0,.225) and +(0,-.225).. (\tikzchildnode.south)}]
            \node {a}[grow'=up]
              child { node {b} 
                child { node[thick,fill=gray!50] {c} 
                  child { node[draw=gray] {\color{gray}d}}
                }
                child { node[thick,fill=gray!50] {e}
                  child[missing]
                  child { node[draw=gray] {\color{gray}f}}
                  child { node[draw=gray] {\color{gray}g}}
                }
                child { node[thick,fill=gray!50] {h}
                }
              }
              child[missing]
              child[missing]
              child[missing]
              child { node[thick,fill=gray!50] {i}
                child { node[draw=gray] {\color{gray}j}
                  child { node[draw=gray] {\color{gray}k} }
                  child[missing]
                  child[missing]
                }
                child { node[draw=gray] {\color{gray}l} 
                  child { node[draw=gray] {\color{gray}m}}
                  child { node[draw=gray] {\color{gray}n}}
                  child { node[draw=gray] {\color{gray}o}}
                } 
              };
          \end{tikzpicture}
        \end{minipage}

        \begin{minipage}{.5\linewidth}
            \centering
        \begin{tikzpicture}[
            sibling distance=3em,
            level distance= 2em,
            every node/.style = {
              align=center}]]
            \node {$\epsilon$}
              child { node {\tup{b,7}} }
              child { node {\tup{i,7}} }
                ;
          \end{tikzpicture}
        \end{minipage}%
        \begin{minipage}{.5\linewidth}
            \centering
            \begin{tikzpicture}[
                sibling distance=3em,
                level distance= 2em,
                every node/.style = {
                  align=center}]]
                \node {$\epsilon$}
                  child { node {\tup{i,7}} 
                    child { node {\tup{c,2}} }
                    child { node {\tup{e,3}} }
                    child { node {\tup{h,1}} }
                  }
                    ;
              \end{tikzpicture}
            \end{minipage}
    \caption{Illustration of the argument used for 
    Theorem~\ref{th:OP-tree-exp-binary-PSpace-mem}. On the top, the partially 
    guessed proof tree for two consecutive steps of the algorithm is shown, where the dark 
    nodes are what is currently kept in memory.
    On the bottom, we see how the corresponding tuples are organized into a tree 
    satisfying Conditions~\ref{pspace-root}--\ref{pspace-sum}.
    }
    \label{fig:illustration-tree-algorithm}
\end{figure}

%% file: conclusion.tex
\section{Conclusion}\label{sec:conclusion}

We have investigated the complexity of finding optimal proofs \wrt quality measures that satisfy the property of being \emph{monotone recursive}. Two important examples of this class of measures, \emph{depth} and \emph{tree size}, have been considered in detail in combination with exponential and polynomial \reasoners. The obtained results are promising: given a \reasoner, the search for an optimal proof for an entailment can be easier than producing all of the proofs by this \reasoner.
The algorithms used to show the upper bounds can serve as building blocks for finding an optimal proof \wrt to a monotone recursive measure automatically.

We conjecture that weighted versions of \emph{tree size} and \emph{depth}, where sentences or inference steps can have associated rational weights, are also monotone recursive, and the generic upper bounds established in this paper can be straightforwardly applied to them. However, a more thorough study is required here, since the complexity of the decision problem depends on the exact way in which the weights are employed. This step towards weighted measures is motivated by user studies~\cite{DBLP:journals/kbs/HorridgeBPS13,DBLP:conf/semweb/AlharbiHSHT17,DBLP:conf/ekaw/NguyenPPW12}, demonstrating that different types of sentences and logical inferences can be more or less difficult to understand.

%% file: appendix.tex
\clearpage

\appendix

\section{Appendix}

\subsection{Hypergraphs}
\begin{definition}[Hypergraph]
    A \emph{(finite, directed, labeled) 
    hypergraph}~\cite{DBLP:journals/cor/NielsenAP05} is a 
    triple $H=(V,E,\el)$, where
    \begin{itemize}
        \item $V$ is a finite set of \emph{vertices},
        \item $E$ is a set of \emph{hyperedges}~$(S,d)$ with \emph{source vertices} 
        $S\subseteq V$
        and
        \emph{target vertex} $d\in V$, and
        \item $\el\colon V\to \mathcal{S}_\L$ is a \emph{labeling function} that assigns 
        sentences to
        vertices.
    \end{itemize}
\end{definition}
We extend the function~$\ell$ to hyperedges as follows: $\ell(S,d):=\big(\{\ell(s)\mid 
s\in 
S\},\ell(d)\big)$.
The \emph{size} of~$H$, denoted~$|H|$, is measured by the size of the labels of
its hyperedges: $$\card{H}:=\sum_{(S,d)\in E}\card{(S,d)}, \text{ where } 
\card{(S,d)}:=\card{\el(d)}+\sum_{v\in
    S}\card{\el(v)}.$$
A vertex $v\in V$ is called a \emph{leaf} if it has no incoming hyperedges, \ie there is no
$(S,v)\in E$; and $v$ is a \emph{sink} if it has no outgoing hyperedges, \ie there is no 
$(S,d)\in
E$ such that $v\in S$. We denote the set of all leaves and the set of all sinks in~$H$ as
$\mathit{leaf}(H)$ and $\mathit{sink}(H)$, respectively.

A hypergraph $H'= (V', E', \el')$ is called a \emph{subgraph} of $H= (V, E, \el)$ if 
$V'\subseteq 
V$, $E'\subseteq E$ and $\el'=\el|_{V'}$.
In this case, we also say that $H$ \emph{contains} $H'$ and write $H' \subseteq H$.
Given two hypergraphs $H_1=(V_1,E_1,\el_1)$ and $H_2=(V_2,E_2,\el_2)$ \st
$\el_1(v)=\el_2(v)$ for every $v\in V_1\cap V_2$, the \emph{union} of the two 
hypergraphs is
defined as $H_1 \cup H_2:=$ $(V_1\cup V_2,E_1\cup E_2, \el_1 \cup \el_2)$.

\begin{definition}[Cycle, Tree]
    Given a hypergraph $H=(V,E,\el)$ and $s,t\in V$, a \emph{path} $P$ of length $q\geq 
    0$ 
    in~$H$ from~$s$ to~$t$ is a sequence of vertices and
    hyperedges
    \[ P=(d_0,(S_1,d_1),d_1,(S_2,d_2),\dots, d_{q-1},(S_q,d_q),d_q), \]
    where $d_0=s$, $d_q=t$, and $d_{j-1}\in S_j$ for all $j$, $1\le j\le q$. By $|P|$ we denote the length of a path $P$. 
    If
    there is such a path of length $q > 0$ in~$H$, we say that $t$ is \emph{reachable} 
    from~$s$
    in~$H$.
    If $t = s$, then $P$ is called a \emph{cycle}.
    The hypergraph~$H$ is \emph{acyclic} if it does not contain a cycle.
    The hypergraph~$H$ is \emph{connected} if
    every vertex is connected to every other vertex by a series of paths and reverse paths.

    A hypergraph $H=(V,E,\el)$ is called a \emph{tree} with \emph{root} $t\in V$ if $t$ is 
    reachable
    from every vertex $v\in V\setminus\{t\}$ by exactly one path. In particular, the root is 
    the only
    sink in a tree, and all trees are acyclic and connected.
\end{definition}

\begin{definition}[Homomorphism]\label{def:homomorphism}
    Let $H=(V,E,\el)$, $H'=(V',E',\el')$ be two hypergraphs.  A \emph{homomorphism} 
    from $H$ 
    to $H'$, denoted $h\colon H\rightarrow H'$, is a mapping $h\colon V\to V'$ s.t.\ for all 
    $(S,d)\in E$, one has $h(S,d):=(\{h(v)\mid v\in S\},h(d))\in E'$ and, for all $v\in V$, it 
    holds 
    that $\el'(h(v))=\el(v)$.
    Such an~$h$ is an \emph{isomorphism} if it is a bijection, and its inverse, $h^-\colon 
    H' \to 
    H$,
    is also a homomorphism.
\end{definition}

\begin{definition}[Hypergraph Unraveling]\label{def:unraveling}
    The \emph{unraveling} of an acyclic hypergraph $H=(V,E,\ell)$ at a vertex $v\in V$ is 
    the 
    tree $H_T=(V_T,E_T,\ell_T)$, where $V_T$ consists of $v$ as well as all paths in~$H$ 
    that 
    end in~$v$, $E_T$ contains all hyperedges $(\{P_1,\dots,P_n\},P)$ (resp.\ 
    $(\{P_1,\dots,P_n\},v)$) where each~$P_i$ is of the form $(d_i,(S,d))\cdot P$ (resp.\ 
    $(d_i,(S,v),v)$) such that $S=\{d_1,\dots,d_n\}$, $\ell_T(v)=\ell(v)$ and $\ell_T(P)$ is 
    the 
    label of the starting vertex of~$P$ in~$H$.
    
    Moreover, the mapping $h_T\colon V_T\to V$ that maps each path to its starting vertex 
    and 
    $v$ to itself is a homomorphism from~$H$ to~$H_T$.
\end{definition}
The tree in Figure~\ref{fig:ex:proof-tree-2} represents the unraveling of the 
hypergraph 
from Figure~\ref{fig:ex:proof-graph-2}.

\subsection{Additional Proofs}

\finitePathsSinkConclusion*
\begin{proof}
    The first statement trivially follows from the acyclicity and the only sink $v_\eta$ in \p. 
    The
    length of a path in \p can be bounded by~$|V|$.
    
    The second claim can be shown by an induction on the depth of \p. Namely, for every 
    $k$
    and every $w\in V$ \st all paths leading to~$w$ have length at most~$k$ it holds that
    $\Tmc\models \el(w)$. The induction base follows from the fact that the leaves are 
    labeled
    with the sentences from \Tmc. For the induction step, for a vertex $w\in V$, we 
    consider an
    hyperedge $(S,w)\in E$. Every $s\in S$ satisfies the induction hypothesis and, thus,
    $\Tmc\models \el(s)$. By \p being a derivation structure, it holds $\{\el(s) | s\in S\}
    \models \el(w)$ and, by transitivity of model-based entailment, $\Tmc\models \el(w)$.
    \qed
\end{proof}

We now show that Definition~\ref{def:measure} is more general than the similar definition of $\Psi$-measures in~\cite{DBLP:conf/dlog/AlrabbaaBBKK20}, and in particular now also covers the measure \emph{depth}.

\begin{definition}[\!\!\cite{DBLP:conf/dlog/AlrabbaaBBKK20}]
A measure \m is a \emph{$\Psi$-measure} if, for every $\p\in \mathrm{P}_{\L}$,
\begin{enumerate}[leftmargin=2.2em]
\item[\pol] $\m(\p)$ is computable in \textbf{\underline{p}}olynomial time in the size of~$\p$,
\item[\textbf{\upshape{[SI]}}] every \textbf{\underline{s}}ubproof of a homomorphic \textbf{\underline{i}}mage of~$\p$ weighs no more than~$\p$, \ie $\m(\p'')\le\m(\p)$ for any homomorphism $h\colon \p\to \p'$ and $\p''\subseteq h(\p)$ such that $\p''\in\mathrm{P}_{\L}$.
\end{enumerate}
\end{definition}

\begin{restatable}{lemma}{phiandpsi}\label{lem:phi-and-psi}
    For proofs according to Definition~\ref{def:proof}, every $\mathrm{\Psi}$-measure (as introduced in~\cite{DBLP:conf/dlog/AlrabbaaBBKK20}) is a \pis-measure, but not vice versa.
\end{restatable}
\begin{proof}
Trivially follows since \si requires that only minimal subproofs of the homomorphic image weighs no more than \p. 
However, in contrast to~\cite{DBLP:conf/dlog/AlrabbaaBBKK20}, in this paper we require every vertex in a proof in Definition~\ref{def:proof} to have at most one incoming edge. Thus, rigorously speaking, measures in this paper may be undefined for some proof hypergraphs from the paper~\cite{DBLP:conf/dlog/AlrabbaaBBKK20}.

Moreover, \emph{depth} is a \pis-measure (see Lemma~\ref{lem:table-psi}) but not a $\Psi$-measure (see Lemma~8 in~\cite{DBLP:conf/dlog/AlrabbaaBBKK20}).
\qed
\end{proof}

For the following proof, we define $\p^{-v}=(V^{-v},E^{-v},\el^{-v})$ as the largest subgraph of~$\p$ such that
\begin{itemize}
\item $v \in \mathit{leaf}(\p^{-v})$,
\item $\mathit{leaf}(\p^{-v})\subseteq\mathit{leaf}(\p)$, and
\item $\mathit{sink}(\p^{-v})=\mathit{sink}(\p)$.
\end{itemize}

Intuitively, we obtain $\p^{-v}$ by removing the proof of $v$, \ie$\p_v$, from \p. Therefore, for
every $w\in V_v$ where $w\neq v$, and every $(S,d) \in E_v$ where $w \in S$, if all paths
$P$
in \p from $w$ to $\mathit{sink}(P)$ go
through $v$, then $w \not\in V^{-v}$ and $(S,d)\not\in E^{-v}$, otherwise, $w \in V^{-v}$ and
$(S,d)\in E^{-v}$.
$\p^{-v}$ need not be a proof \wrt~\Tmc since $v$ is now a leaf, but $\ell(v)$ may not be a 
sentence from~\Tmc.

\LemAdmHyperproofInside*
\begin{proof}

    Let \p be such a proof with associated homomorphism $h\colon \p\to\R(\Tmc,\eta)$.

    First, we show that there is a subproof for $\Tmc\models\eta$  in the homomorphic image.
    If $h(\p)$ is acyclic and every vertex has at most one incoming edge, then we already 
    found 
    one subproof.
    Since \p has a unique sink~$v_\eta$, it must be mapped to a unique sink $h(v_\eta)$
    in~$h(\p)$, and thus $h(\p)$ is the desired subproof of~$\R(\Tmc,\eta)$.

    If $h(\p)$ is not acyclic or there is a vertex with more than one incoming edge, our
    goal is to
    find another admissible proof~$\p^*$ \wrt
    $\R(\Tmc,\eta)$ that uses a subset of the vertices of~\p such that $h(\p^*)\subseteq 
    h(\p)$
    is acyclic with only one incoming edge for any vertex. %
    For this purpose, first consider an arbitrary cycle in~$h(\p)$, which must be due to two 
    vertices
    $v,v'$ in~\p such that $h(v)=h(v')$ and there is a path between $v$ and~$v'$ (or due 
    to
    multiple such pairs of vertices).
    Since \p is acyclic, we can assume that there is a path from~$v$ to~$v'$, but no path
    from~$v'$ to~$v$. We now consider the two subproofs~$\p_v$ and~$\p_{v'}$. As 
    there is a
    path from $v$ to $v'$, we have $\p_v\subset\p_{v'}$. Since $h(v)=h(v')$, both vertices 
    are
    labeled with the same sentence. The idea of the following construction is to remove 
    $\p_{v'}$
    from~\p and replace it with~$\p_v$, which effectively removes all paths from~$v$ 
    to~$v'$.
    
    More formally, we first consider the hypergraph $H=\p^{-v'}\cup\p_v$ and then, in the
    hyperedges $(S,d)$ in~$H$ that still contain $v'\in S$, we replace $v'$ by~$v$, 
    effectively
    merging the two vertices, remove $v'$ from the set of vertices, and thus obtain a 
    hypergraph
    $\p'$.
    If there was no such hyperedge, then $v'$ was the sink of~\p, \ie $\el(v')=\eta$, and 
    $v$ will
    now be the new sink in~$\p'$ with $\el(v)=\el(v')=\eta$.

    For a vertex $w$ in $h(\p)$ with more than one incoming edge, again there
    must be
    two vertices $v,v'$ in \p \st $h(v)=h(v')=w$. We can thus apply the same procedure as
    above.
    However, since $\p_v\subset\p_{v'}$ may not hold, it does not matter which of the subproofs $\p_v,\p_{v'}$ is replaced by the other.

    We now show that $\p'$ is also an admissible proof \wrt $\R(\Tmc,\eta)$.
    Our construction does not produce new leaves, and hence $\p'$ is still grounded.
    Clearly, all remaining edges are sound since they were already sound in~\p.
    Moreover, $\p'$ is acyclic and every vertex has only one incoming edge since all edges 
    and 
    cycles in~$\p'$ can be traced back to paths in~\p that
    involve both~$v$ and~$v'$; but we have assumed that there are no paths from~$v'$
    to~$v$, and have destroyed all paths from~$v$ to~$v'$.
    As argued above, we have also kept the property that there is exactly one sink, which is
    labeled with~$\eta$.
    Observe that $h$ is also a homomorphism from~$\p'$ to $\R(\Tmc,\eta)$ (when 
    restricted to
    the vertices of~$\p'$), because $h(v)=h(v')$, and moreover $h(\p')\subseteq h(\p)$.

    This means that, after finitely many such operations, we can obtain from~\p the desired
    proof~$\p^*$ such that $h(\p^*)\subseteq h(\p)$ is acyclic with every vertex having at 
    most 
    one incoming edge.
    Since $h(\p^*)$ also has a unique sink labeled by~$\eta$, it is a subproof with sink
    $h(v_\eta)$ in
    $\R(\Tmc,\eta)$.

    Second, consider the set $\mathbf{Q}$ of all possible subproofs with sink $h(v_\eta)$ in
    $\R(\Tmc,\eta)$. As shown above, $\mathbf{Q}$ is non-empty. Thus, if the proof
    $h(\p^*)$ obtained in the previous step is minimal \wrt \m, then
    $\m(h(\p^*))\le\m(\p)\le q$
    by~\si in Definition~\ref{def:measure}. Otherwise, there is another subproof $\q 
    \in\mathbf{Q}$
    with sink $h(v_\eta)$ in $\R(\Tmc,\eta)$, \st $\m(\q) < \m(h(\p^*))$ and its weight is
    minimal \wrt
    \m among
    $\mathbf{Q}$. Then, again, $\m(\q)\le\m(\p)\le q$ by~\si in 
    Definition~\ref{def:measure}.
    \qed
\end{proof}

An interesting property of recursive measures is that they are invariant under unraveling (see Definition~\ref{def:unraveling}).

\begin{restatable}{lemma}{proofMeasureEqualItsUnraveling}\label{lem:recursive-unraveling}
    Let \m be a recursive \pis-measure, \p be a proof for $\Tmc\models\eta$ and $\p_T$ its unraveling into a tree (starting at the sink). Then $\m(\p)=\m(\p_T)$.
\end{restatable}
\begin{proof}
    We show this by induction on the depth of~\p.
    If \p contains only one vertex, then $\p_T=\p$, and thus the claim is trivial.
    If the longest path in~\p has length~$n$, assume that the claim holds for all proofs of 
    depth 
    at most~$n-1$.
    Consider the unique hyperedge $(S,v)\in\p$ that leads to the sink~$v$ of~\p.
    Then $\p_T$ is isomorphic to the union of the unravelings~$\p_{w,T}$ of all~$\p_w$ 
    with 
    $w\in S$, together with the hypergraph that contains only the edge~$(S,v)$ (if we 
    identify 
    the paths $(w,(S,v),v)$ with the vertices~$w$).
    By induction, $\m(\p_w)=\m(\p_{w,T})$ since each $\p_w$ is of depth at most~$n-1$.
    We obtain that
    \begin{align*}
    \m(\p)
    &= \edge_\m\big(\ell(S,v),\{\m(\p_w)\mid w\in S\}\big) \\
    &= \edge_\m\big(\ell(S,v),\{\m(\p_{w,T})\mid w\in S\}\big) \\
    &= \m(\p_T)
    \end{align*}
    since \m is recursive.
    \qed
\end{proof}

\LemTableIsPSI*
\begin{proof}
    Both measures are monotone recursive by definition and can be computed in polynomial time in the size of the input proof.
    
    For \emph{tree size}, we consider a homomorphism $h\colon\p\to H$ and a vertex~$w$ in~$\p$. By the procedure described in the proof of Lemma~\ref{lem:adm-hyperproof-inside}, there exists a proof $\p^*_w$ with edges from $\p$ (modulo renamed vertices) \st $h(\p^*_w) \subseteq h(\p)$ is a proof with sink $h(w)$. It is not hard to see that, by construction, $\mtree(\p^*_w) \leq \mtree(\p)$, since we replace subproofs $\p_{v'}$ with $\p_v$, where $\p_v \subset \p_{v'}$ and $\mtree(\p_v) < \mtree(\p_{v'})$. Since $h(\p^*_w)$ is a proof, $\mtree(h(\p^*_w))$ is defined. The property of any proof vertex having only one incoming edge guarantees that $h(\p^*_w)$ and $\p^*_w$ are isomorphic. Since homomorphisms preserve edges, $\mtree(\p^*_w)=\mtree(h(\p^*_w))$.
    Thus, for every vertex~$w$ in~\p, there is a proof with sink~$h(w)$ in~$h(\p)$ of tree size no greater than $\mtree(\p)$. Trivially, every minimal (\wrt \mtree) subproof with sink~$h(w)$ weighs no more than $\mtree(h(\p^*_w))$.
    Every vertex in $h(\p)$ has a pre-image in $\p$ and, therefore, \si holds for \emph{tree size}.

    For \emph{depth}, we can use a similar argument. The process of replacing subproofs $\p_{v'}$ with their smaller alternatives $\p_v \subset \p_{v'}$ also results in a non-increasing depth for $\p^*_w$. As a consequence, $\mdepth(h(\p^*_w)) \leq \mdepth(\p)$ for $h(\p^*_w)\subseteq h(\p)$.
    \qed
\end{proof}

\LemDijkstra*
\begin{proof}
    We can show the following facts about this algorithm.
    \begin{enumerate}[label=(\Roman*),leftmargin=*]
        \item\label{p:sound}%
        \emph{Whenever $\p(v)$ is defined, then it is a proof for~$\ell(v)$ contained in
            $\R(\Tmc,\eta)$.}
        
        \smallskip
        We prove this by induction on the order in which the hypergraphs~$\p(v)$ are 
        constructed
        by Algorithm~\ref{alg:dijkstra}.
        The ones in Line~\ref{l:tbox} consist of a single leaf~$v$, which is labeled by a theory sentence,
        and hence are sound, grounded, acyclic, and have the single sink~$v$.
        Similarly, $\p(v)$ in Line~\ref{l:tautology} is always a proof since it consists of a 
        single
        edge from $\R(\Tmc,\eta)$, has no leaves, and has~$v$ as the only sink.
        
        Consider now the hypergraph~\p constructed in Line~\ref{l:new-proof} as a 
        possible
        candidate for $\p(v)$ (where $v=d$).
        At this point, all $\p(s)$, $s\in S$, are already defined since the counter~$k(e)$ can 
        only
        reach $|S|$ if each $s\in S$ has already been chosen in Line~\ref{l:choose}, and 
        thus
        $\p(s)$ must have been defined.
        Hence, by induction, each $\p(s)$ is a proof for~$\ell(s)$ contained in 
        $\R(\Tmc,\eta)$,
        and,
        because we assume that $\R(\Tmc,\eta)$ contains no two vertices with the same 
        label,
        must
        have $s$ as sink.
        This shows that the hypergraph~\p constructed in Line~\ref{l:new-proof} is sound,
        grounded,
        and has a single sink, namely~$v$.
        Finally, $\p(v)$ is only updated to~$\p$ in Line~\ref{l:update} if $\p$ is acyclic and
        therefore it is a proof.

        \item\label{p:monotone} %
        \emph{If vertex~$v$ is chosen before vertex~$w$ in Line~\ref{l:choose}, then
            $\m(\p(v))\le\m(\p(w))$.}
        
        \smallskip
        We show that after choosing~$v$ in Line~\ref{l:choose} the algorithm cannot 
        produce a
        new proof~\p in Line~\ref{l:new-proof} with $\m(\p)<\m(\p(v))$, and thus the 
        smallest
        weight $\min\{\m(\p(w))\mid w\in Q\}$ can never decrease.
        Consider the proof~\p from Line~\ref{l:acyc}.
        Since $v\in S$, we have $\p(v)\subset\p$, and therefore $\m(\p(v))\le\m(\p)$ 
        by~\si.

        \item\label{p:terminating} %
        \emph{Algorithm~\ref{alg:dijkstra} terminates in polynomial time.}
        
        \smallskip
        Item~\ref{p:monotone} implies that each vertex $v\in V$ can be removed from~$Q$ 
        at
        most once:
        in order for $v$ to be added again to $Q$ in Line~\ref{l:update}, there would need 
        to exist
        a
        proof other than~$\p(v)$ with the same sink~$v$ but a smaller weight, but
        according to \ref{p:monotone}, after choosing~$v$ in Line~\ref{l:choose}, the
        algorithm does not construct any proofs with a weight smaller than $\m(\p(v))$ (for 
        any
        sink).
        Therefore, during the complete run of the algorithm, each edge $e=(S,d)\in E$ will 
        be used
        at most once in Line~\ref{l:new-proof}.
        Moreover, all primitive operations in the algorithm can be done in polynomial time, 
        such as
        checking acyclicity of hypergraphs in Line~\ref{l:acyc} or finding the minimal value
        $\m(\p(v))$ in Line~\ref{l:choose}.
        It follows that Algorithm~\ref{alg:dijkstra} terminates in time polynomial in the size
        of~$\R(\Tmc,\eta)$.
        
        \item\label{p:complete} %
        \emph{Every vertex~$v\in V$ that is the sink of a proof~\p contained in 
        $\R(\Tmc,\eta)$ is
            added to~$Q$ at some point.}
        
        \smallskip
        We prove this by induction on the structure of~\p.
        If \p contains only~$v$, then either $\ell(v)\in\Tmc$, and hence $v$ is added 
        to~$Q$ in
        Line~\ref{l:tbox}, or otherwise there is an edge $(\emptyset,v)$ in~\p (and hence 
        in~$E$),
        in which case $v$ is added to~$Q$ in Line~\ref{l:tautology}.
        
        If \p has more than one vertex, then it must contain at least one edge $e=(S,v)\in E$, 
        where
        each $s\in S$ is the sink of a subproof of~\p in $\R(\Tmc,\eta)$.
        By induction we know that each $s\in S$ is added to~$Q$ at some point during the
        algorithm.
        By Item~\ref{p:terminating}, they must also be removed from~$Q$ at some point
        afterwards, and hence eventually $k(e)$ reaches~$|S|$ in Line~\ref{l:counter}.
        If $\p(v)$ was already defined at this point, then $v$ had already been added to~$Q$
        earlier. Otherwise, $v$ is now added to~$Q$ in Line~\ref{l:update}.
        
        \item\label{p:optimal} %
        \emph{When Algorithm~\ref{alg:dijkstra} terminates and $\p(v)$ is defined, then
            $\m(\p(v))$ is minimal among
            all proofs for $\ell(v)$ contained in $\R(\Tmc,\eta)$.}

        \smallskip
        By~\ref{p:sound}, $\p(v)$ is a proof of this form.
        Assume to the contrary that there is a proof~\p for~$\ell(v)$ contained in 
        $\R(\Tmc,\eta)$
        such that $\m(\p)<\m(\p(v))$.
        Then~\p and $\p(v)$ must both have the sink~$v$, because we assume that
        $\R(\Tmc,\eta)$
        contains no two vertices with the same label.
        Assume moreover that
        \begin{enumerate}[label=\roman*)]
            \item\label{i:optimal} $\p$ is an optimal proof for~$\ell(v)$ in $\R(\Tmc,\eta)$, 
            that is,
            $\m(\p)
            \leq\m(\p')$ for every other proof $\p'$ for~$\ell(v)$ in $\R(\Tmc,\eta)$ (\cf
            Corollary~\ref{cor:existence}), and
            \item\label{i:smallest} among all other vertices $v'\in V$ and all proofs $\p'$
            for~$\ell(v')$ in $\R(\Tmc,\eta)$ such that $\m(\p')<\m(\p(v'))$, we
            also have $\m(\p)\le\m(\p')$ and whenever $\m(\p)=\m(\p')$, then $|\p|\le|\p'|$.
        \end{enumerate}
        Consider the unique last inference step~$(S,v)$ in~\p.
        We show that, for every vertex $w\in S$, an optimal proof was assigned to $\p(w)$ 
        before 
        $v$ was chosen in Line~\ref{l:choose}.
        
        First, $\p_w$ is a proof of $\ell(w)$ contained in $\R(\Tmc,\eta)$, and hence by
        \ref{p:complete} $\p(w)$ must be defined when Algorithm~\ref{alg:dijkstra} 
        terminates.
        
        Next we show that $\m(\p(w))\le\m(\p_w)$.
        Assume to the contrary that $\m(\p_w)<\m(\p(w))$, \ie $w$ and $\p_w$ satisfy the
        precondition in Assumption~\ref{i:smallest}, and thus $\m(\p)\le\m(\p_w)$ must 
        hold.
        However, by~\si, we have $\m(\p_w)\le\m(\p)$ (observe that $\m(\p_w)$ must be minimal among the subproofs of~$\p_w$ with sink~$w$ since otherwise $\m(\p)$ would not be minimal, because \m is monotone recursive).
        Thus, Assumption~\ref{i:smallest} also
        yields that $|\p|\le|\p_w|$, which contradicts the fact that $\p_w$ is a subproof of~\p.
        
        We obtain that
        \[ \m(\p(w))\le\m(\p_w)\le\m(\p)<\m(\p(v)), \]
        which by \ref{p:monotone} means that $w$ must have been chosen (in
        Line~\ref{l:choose})
        before~$v$.
        
        To summarize, for every vertex $w\in S$, we know that an optimal proof 
        for~$\ell(w)$ with
        weight~$\m(\p_w)$ has already been assigned to $\p(w)$ before $v$ is chosen in
        Line~\ref{l:choose}, and moreover each $w\in S$ was chosen before~$v$. But then, 
        for
        one of
        these vertices~$w$ (the last one to be processed), a proof $\p'$ is constructed from 
        the
        subproofs $\p(w)$ and the edge $(S,v)$ in Line~\ref{l:new-proof}.
        Moreover, since $\m(\p(w))\le\m(\p_w)$ holds for all $w\in S$ and \m is a 
        monotone 
        recursive measure, we obtain
        \begin{align*}
        \m(p')
        &= \edge_\m\big(\ell(S,v),\{\m(\p(w))\mid w\in S\}\big) \\
        &\le \edge_\m\big(\ell(S,v),\{\m(\p_w)\mid w\in S\}\big) \\
        &= \m(\p).
        \end{align*}
        Since $\p'$ was constructed as a candidate for~$\p(v)$ by the algorithm, we further 
        know 
        that $\m(\p(v))\le\m(\p')\le\m(\p)$, which contradicts our initial assumption that 
        $\m(\p)<\m(\p(v))$.
        We obtain that $\p(v)$ must be optimal.
    \end{enumerate}
    
    Since by Lemma~\ref{lem:adm-hyperproof-inside} the derivation structure 
    $\R(\Tmc,\eta)$
    contains an optimal proof for $\Tmc\models\eta$,
    Items~\ref{p:terminating}--\ref{p:optimal}
    show that Algorithm~\ref{alg:dijkstra} returns such a proof in Line~\ref{l:return}.
    \qed
\end{proof}

\treePSapceUpperbound*
\begin{proof}[continued]
    By our assumption that all sentences in $\R(\Tmc,\eta)$ are of polynomial size and due to the binary encoding of numbers, each tuple in~$S$ takes only polynomial space. It thus remains to verify the properties \ref{pspace-root}--\ref{pspace-sum} to show that $S$ is polynomially bounded.

    We proceed by induction on the algorithm
    steps:
    for the original set, $\tup{\eta,q}$ is the only child of the root. In each step where we 
    replace
    $\tup{\eta',q'}$ with $\tup{\eta_1,q_1}$, $\ldots$, $\tup{\eta_m,q_m}$, we remove 
    from the
    current
    tree the node labeled with $\tup{\eta',q'}$, and if $m=1$, we replace it with the new 
    node,
    and if
    $m>1$, we put the new nodes under the leaf labeled with the now smallest value, 
    or under $\epsilon$ if there is no other leaf. It is 
    clear
    that the
    resulting structure must always satisfy~\ref{pspace-root}, \ref{pspace-labels} 
    and~\ref{pspace-sum}. It satisfies~\ref{pspace-two} because we replace a node in 
    the tree 
    if the
    number of new nodes is~$1$. It satisfies~\ref{pspace-p} because we only replace 
    nodes or 
    add $m>1$ nodes to a leaf,
    and we
    have $m\leq p$. For~\ref{pspace-half}, we note that because of~\ref{pspace-two} 
    and~\ref{pspace-sum}, for each node labeled
    $\tup{\eta',q'}$, its child $\tup{\eta'',q''}$ with minimal~$q''$ must always satisfy
    $q''<\frac{q'}{2}$, and we always put new nodes only under the node with the smallest
    associated number. Now, as a consequence of~\ref{pspace-half} and because every number 
    in~$S$ 
    is bounded by~$q$, we
    obtain that the subtree containing all inner nodes is a path
    and of depth $\log_2 q$. Together with~\ref{pspace-p}, this implies that the
    tree has at most $p\cdot \log_2 q$ nodes, and with~\ref{pspace-labels}, we obtain 
    that
    $\left|S\right|\leq p\cdot\log_2 q$ is always satisfied.
    \qed
\end{proof}

\treePSpaceLowerBound*
\begin{proof}[continued]
    It remains to find the upper bound~$q$ on the tree size of a proof for $\Tmc\models\exFont{Start}\sqsubseteq\exFont{Accept}$.
    
    We first note that $T$ has at most $\lvert Q\rvert\cdot\lvert\Gamma\rvert^k$ 
    different
    configurations, and thus an accepting run involves $m\leq \lvert
    Q\rvert\cdot\lvert\Gamma\rvert^k$ steps. Let $\exFont{Conf}_0$, $\ldots$,
    $\exFont{Conf}_m$ denote
    the sequence of conjunctions representing those configurations, where 
    $\exFont{Conf}_0$
    represents
    the initial configuration and $\exFont{Conf}_m$ the final one. For $j\in\{0,\ldots,m-1\}$, 
    the
    transitions from one configuration to the succeeding is encapsulated in an entailment of 
    the
    form $\exFont{Conf}_j\sqsubseteq\exists r.\exFont{Conf}_{j+1}$. Those
    entailments are inferred as follows with the calculus:
    \begin{itemize}
        \item $(2+k+1)+k=2k+3$
        times we have to apply $\mathsf{CR1}$ and then
        $\mathsf{CR2}$
        to obtain
        sentences as in~\eqref{al:transition-function} and~\eqref{al:keep-tape} but with the
        entire
        configuration encoding $\exFont{Conf}_j$ on the left-hand side. For each of the $2+k+1$ 
        sentences corresponding to~\eqref{al:transition-function}, this gives a tree of size 5 (3 times 
        we apply $\mathsf{CR1}$ without a premise to get the sentences $\exFont{Conf}_j\sqsubseteq 
        A$, where $A$ is an atom on the left-hand side of \eqref{al:transition-function}, followed by one 
        application of $\mathsf{CR2}$ with a sentence corresponding 
        to~\eqref{al:transition-function}). For the $k$ sentences in~\eqref{al:keep-tape}, this requires a 
        tree of size 4 (the same argument as before, but now with one atom less on the left-hand side 
        of~\eqref{al:keep-tape}). 
        
        \item The resulting $2k+3$
        sentences are then step-wise combined using $\mathsf{CR4}$ to
        obtain
        the desired entailment with $\exFont{Conf}_{j+1}$ under the role restriction. 
        This generates $2k+1$ intermediate conclusions. Together with the final conclusion
        $\exFont{Conf}_j\sqsubseteq\exists r.\exFont{Conf}_{j+1}$, this makes $2k+2$ additional tree vertices in total.
      \end{itemize}
    Consequently, each inference of $\exFont{Conf}_j\sqsubseteq\exists 
    r.\exFont{Conf}_{j+1}$
    is
    generated by a proof of tree size $(2+k+1)\cdot 5 + k\cdot 4 + (2k+2) =11k+17$. 

    The complete tree proof for $\exFont{Start}\sqsubseteq\exFont{Accept}$ is now obtained by
    first
    generating $\exFont{Conf}_m\sqsubseteq\exFont{Accept}$ for the final configuration. 
    This
    involves 3
    vertices, the first by using $\mathsf{CR1}$ to generate $\exFont{Conf}_m\sqsubseteq 
    S_f$,
    where $f$ is the
    accepting state of the configuration, and then using $\mathsf{CR2}$ with 
    Sentence~\eqref{al:final} to generate
    $\exFont{Conf}_m\sqsubseteq\exFont{Accept}$. From here, we follow the sequence 
    of
    configurations
    backwards, each time inferring from $\exFont{Conf}_j\sqsubseteq\exFont{Accept}$ 
    and
    \eqref{al:accept} the sentence $\exFont{Conf}_j\sqsubseteq\forall r^-.\exFont{Accept}$ 
    (increasing the tree size by~$2$), and
    then using this sentence together with $\exFont{Conf}_{j-1}\sqsubseteq\exists 
    r.\exFont{Conf}_j$ via
    $\mathsf{CR3}$ to get $\exFont{Conf}_{j-1}\sqsubseteq\exFont{Accept}$ (tree size 
    increased by 
    $11k+17$).
    Finally, from 
    $\exFont{Conf}_0\sqsubseteq\exFont{Accept}$ we 
    get
    to
    $\exFont{Start}\sqsubseteq\exFont{Accept}$ using the $2k+3$
    sentences corresponding to~\eqref{al:initial} and $\mathsf{CR2}$. We obtain that the entire tree 
    proof requires 
    at
    most
    \begin{align*}
    1 + (2k+3) + m\cdot(3 + 2 + (11k+17))
    &\leq (m+1)\cdot(11k+ 22)\\
    &\leq (\lvert Q\rvert \cdot \lvert\Gamma\rvert^{p(\lvert w\rvert)}+1)\cdot (17 + 11 p (\lvert w\rvert))
    \end{align*}
    vertices. Note that this number can be encoded using polynomially many bits.\qed
\end{proof}